\documentclass{article}

\usepackage{microtype}
\usepackage{graphicx}
\usepackage{subcaption}
\usepackage{booktabs} 

\usepackage{hyperref}

\usepackage[preprint]{icml2026}

\usepackage{amsmath}
\usepackage{amssymb}
\usepackage{mathtools}
\usepackage{amsthm}
\usepackage{subcaption}
\usepackage{multirow}

\usepackage[capitalize,noabbrev]{cleveref}

\theoremstyle{plain}
\newtheorem{theorem}{Theorem}[section]

\newtheorem{lemma}[theorem]{Lemma}
\newtheorem{corollary}[theorem]{Corollary}
\theoremstyle{definition}

\newtheorem{assumption}[theorem]{Assumption}
\theoremstyle{remark}
\newtheorem{remark}[theorem]{Remark}

\DeclareMathOperator*{\argmax}{arg\,max}

\DeclareMathOperator*{\col}{Col}
\DeclareMathOperator*{\row}{Row}
\DeclareMathOperator*{\rank}{rank}
\newcommand{\indep}{\perp\!\!\!\!\perp}

\usepackage[textsize=tiny]{todonotes}

\icmltitlerunning{Causal Matrix Completion under Multiple Treatments via Mixed Synthetic Nearest Neighbors}

\begin{document}

\twocolumn[
  \icmltitle{Causal Matrix Completion under Multiple Treatments \\ via Mixed Synthetic Nearest Neighbors}



  \icmlsetsymbol{equal}{*}

  \begin{icmlauthorlist}
    \icmlauthor{Minrui Luo}{equal,thu}
    \icmlauthor{Zhiheng Zhang}{equal,shufe}
  \end{icmlauthorlist}

  \icmlaffiliation{thu}{Institute for Interdisciplinary Information Sciences, Tsinghua University, Beijing, China}
  \icmlaffiliation{shufe}{School of Statistics and Data Science, Shanghai University of Finance and Economics, Shanghai, China}

  \icmlcorrespondingauthor{Zhiheng Zhang}{zhangzhiheng@mail.shufe.edu.cn}

  \icmlkeywords{Matrix Completion, Causal Inference, Multiple Treatment, Data Combination}

  \vskip 0.3in
]



\printAffiliationsAndNotice{}  

\begin{abstract}


Synthetic Nearest Neighbors (SNN) provides a principled solution to causal matrix completion under missing-not-at-random (MNAR) by exploiting local low-rank structure through fully observed anchor submatrices. However, its effectiveness critically relies on sufficient data availability within each treatment level, a condition that often fails in settings with multiple or complex treatments. 
In this work, we propose Mixed Synthetic Nearest Neighbors (MSNN), a new entry-wise causal identification estimator that integrates information across treatment levels. We show that MSNN retains the finite-sample error bounds and asymptotic normality guarantees of SNN, while enlarging the effective sample size available for estimation. Empirical results on synthetic and real-world datasets illustrate the efficacy of the proposed approach, especially under data-scarce treatment levels. 

\end{abstract}

\section{Introduction}
\label{sec:introduction}

Causal inference from observational data has become a cornerstone of modern data science, enabling the rigorous evaluation of interventions across diverse domains such as economics, public policy, and digital platforms. A critical challenge in this field is the estimation of counterfactual outcomes—i.e., what would have transpired had a different treatment been applied—when faced with incomplete data and complex treatment structures. Recent advancements in matrix completion, a well-established tool from machine learning, have been adapted to this causal inference framework, providing a principled methodology for imputing missing potential outcomes under low-rank structural assumptions. This fusion of causal inference and matrix completion is particularly pertinent in contemporary applications, where treatment regimes are multi-faceted (e.g., varying exposure levels in online advertising or policy intensities) and data are \emph{missing not at random} (MNAR), as the very process of treatment assignment is often driven by unobserved latent factors.

Within this broader context, we focus on the problem of \emph{causal matrix completion under multiple, discrete treatment levels}. This extends the binary treatment framework introduced by \citet{agarwal2023causal} to accommodate a more realistic and complex setting. The problem shifts from the task of completing a single matrix of potential outcomes to addressing a collection of matrices (or equivalently, a three-dimensional tensor), one for each treatment level, where only one outcome per unit is observed. While previous work, such as the \emph{Synthetic Interventions} framework \citep{agarwal2020synthetic}, also considers multi-level treatments, it is primarily designed for panel data with a temporal dimension, aiming to estimate average effects over post-treatment periods. In contrast, our approach abstractly separates the temporal aspect, focusing on \emph{entry-wise} counterfactual estimation in a general matrix completion setting under MNAR. This distinction is crucial: our goal is the fine-grained identification of causal effects at the unit level, rather than aggregate temporal effects. The theoretical complexity of this multi-treatment setting, particularly under MNAR conditions, remains largely unexplored, presenting formidable identification and estimation challenges.

A significant challenge arises when treatment assignments are highly imbalanced—some treatment levels are \emph{data-scarce}. Existing methods, such as the Synthetic Nearest Neighbors (SNN) algorithm, encounter difficulties in this regime. The SNN approach requires the construction of "anchor" rows and columns, exclusively using data from the same treatment level. However, when data for a given treatment level is sparse, constructing sufficiently large and valid anchor sets becomes increasingly improbable, leading to estimation failure. This issue is not merely a data limitation; it reflects a fundamental \emph{inefficiency} in the utilization of available information. The key insight bridging this gap is that, under the assumption of shared latent row factors across treatments, the imputation coefficients required for counterfactual estimation can be \emph{identified} by leveraging data from \emph{multiple} treatment levels. Previous multi-treatment frameworks, including Synthetic Interventions, either fail to exploit this cross-treatment identifiability for entry-wise estimation or are not designed to tackle the general MNAR matrix completion problem that we address.

To overcome these challenges, we propose the \textbf{M}ixed \textbf{S}ynthetic \textbf{N}earest \textbf{N}eighbors (\textbf{MSNN}) algorithm. Our method systematically tackles the data scarcity problem by relaxing the stringent same-treatment requirement inherent in the SNN approach. The key innovation lies in the introduction of \emph{Mixed Anchor Rows (MAR)} and \emph{Mixed Anchor Columns (MAC)}, enabling the imputation coefficient $\beta$ to be estimated from a block of data that spans \emph{multiple treatment levels}, while preserving the target row’s data ($x(d)$) from the treatment level of interest. This is made possible by the shared latent factor assumption, which ensures the identifiability of $\beta$ across treatments. The core technical challenge is managing the heterogeneous scales and variances introduced by mixing treatments, which we address through the careful design of appropriate weights.

The most counterintuitive and theoretically impactful result is the \emph{exponential improvement in sample efficiency} for sparse treatment levels. We demonstrate that, under a Missing Completely at Random (MCAR) treatment assignment, the expected number of usable data subgroups for MSNN, $\mathbb{E}[K_{\text{MSNN}}]$, surpasses that of SNN, $\mathbb{E}[K_{\text{SNN}}]$, by a factor of $\left[\sum_{d'} (p_{d'}/p_d)^{r+1}\right]^c$, where $p_d$ is the observation probability for treatment $d$, and $r, c$ are the sizes of the anchor sets. This exponential improvement significantly enhances the feasibility of causal estimation in data-scarce environments. Importantly, MSNN retains the finite-sample error bounds and asymptotic normality of the original SNN estimator, ensuring that this efficiency gain does not come at the cost of statistical rigor. This challenges the conventional wisdom that estimating effects for a rare treatment necessarily requires more data from that specific treatment. Instead, we show that information from more prevalent treatments can be effectively leveraged to learn about rare treatments via shared latent structures. Our contributions are summarized as follows:

\begin{enumerate}
\item We formalize the problem of entry-wise causal matrix completion under multiple MNAR treatment levels and establish a novel identification result, demonstrating that imputation coefficients can be shared across treatments under a shared latent row factor assumption.
\item We propose the MSNN algorithm, which integrates data across treatment levels via Mixed Anchor Sets. We prove that MSNN retains the desirable statistical properties (finite-sample bound, asymptotic normality) of SNN, while achieving exponential improvements in sample efficiency for sparse treatments.
\item Through simulations and a case study on California's tobacco control policy, we demonstrate that MSNN reliably estimates effects for data-scarce treatments where SNN fails, highlighting its practical applicability.
\end{enumerate}

\section{Preliminaries}

This section formalizes the problem, introducing necessary notations and assumptions.

We define the exposure levels: $\mathcal{L} = \{1,2,\cdots,l\}$ and treatment assignment $\boldsymbol{D}$: for consumer $i \in [m]$ and product $j\in [n]$, the assigned exposure level is $D_{ij} \in \{0\}\cup\mathcal{L}$, where $D_{ij}=0$ denotes no exposure. Potential outcomes are defined as $\boldsymbol{Y}$: for consumer $i$ and product $j$, under exposure level $d$ ($d \in \mathcal{L}$), the platform obtains a benefit $Y_{ij}^{(d)} \in \mathbb{R}$. The expected potential outcome is $A_{ij}^{(d)}$, which represents a latent signal matrix assumed to be low-rank across users and items. Specifically, suppose the rank is $r$, then $\boldsymbol{A}^{(d)} = \boldsymbol{U}^{(d)}{\boldsymbol{V}^{(d)}}^\top$, where $\boldsymbol{U}^{(d)} \in \mathbb{R}^{m\times r}$, $\boldsymbol{V}^{(d)} \in \mathbb{R}^{n\times r}$. The zero-mean noise is defined as $\epsilon_{ij}^{(d)} = Y_{ij}^{(d)} - A_{ij}^{(d)}$. We consider general MNAR regime: $\boldsymbol{D} \indep \boldsymbol{Y} \mid \boldsymbol{A}$, meaning that the assignment mechanism depends on the latent attributes but not on the random noise. 
Only the observed outcome $\tilde{\boldsymbol{Y}}$ under the assigned treatment $\boldsymbol{D}$ is observed: 
$
        \tilde{Y}_{ij}
        = 
        \begin{cases}
            Y_{ij}^{(D_{ij})} &,\,D_{ij} \ge 1 \\
            * &,\,D_{ij} = 0 \\ 
        \end{cases} .
$
    Here $*$ means this entry is non-observable (missing), where $\tilde{Y}_{ij} = 0$ under $D_{ij} = 0$ reflects that non-exposed items contribute no benefit. 

    
    
    
    

\textbf{Problem Statement. }Given noisy observation $\tilde{\boldsymbol{Y}}$ with the treatment assignment $\mathbf{D}$, under causal model $\boldsymbol{D} \indep \boldsymbol{Y} \mid \boldsymbol{A}$, we are tasked to find an estimator of the latent expected potential outcome $\boldsymbol{A}$ from $\tilde{\boldsymbol{Y}}$, and provide an entry-wise upper bound of error of estimation and ground truth: $\left\| \hat{\boldsymbol{A}} - \boldsymbol{A} \right\|_{\infty}$. 


\textbf{Assumptions.} Below we introduces basic assumptions for our low-rank multiple matrix completion model. 

\begin{assumption}\label{assumption: low-rank factorization} (Low-rank factorization per layer). For each $(i,j,d) \in \mathbb{R}^{m} \times \mathbb{R}^{n} \times \mathcal{L}$, $
  Y_{i j}^{(d)} = \left\langle u_i^{(d)}, v_j^{(d)} \right\rangle + \epsilon_{i j}^{(d)},
$\end{assumption}

where $u_i^{(d)}, v_j^{(d)} \in \mathbb{R}^{r}$ are latent factors specific to treatment level $d$. This is equal to the expression of $\boldsymbol{Y}^{(d)} = \boldsymbol{U}^{(d)} \boldsymbol{V}^{(d)\top} + \boldsymbol{E}^{(d)}$, where $u_i^{(d)}, v_j^{(d)}$ are the $i^{th}$ and $j^{th}$ rows of $\boldsymbol{U}^{(d)} \in \mathbb{R}^{m \times r}$, $\boldsymbol{V}^{(d)} \in \mathbb{R}^{n \times r}$ respectively.

Below we denote $\boldsymbol{U},\boldsymbol{V}$ and $\boldsymbol{E}$ as the collection of $\boldsymbol{U}^{(d)}, \boldsymbol{V}^{(d)}, \boldsymbol{E}^{(d)}$ over treatment $d$. 

\begin{assumption}\label{assumption: selection on latent factors} (Selection on latent factors). For any treatment $\boldsymbol{D}$, $\mathbb{E}\left[ \boldsymbol{E} \middle | \boldsymbol{U} ,\boldsymbol{V},\boldsymbol{D} \right] = 0 . $


\end{assumption}

\begin{remark}

By taking conditional expectation over $\boldsymbol{D}$, $\mathbb{E}\left[ \boldsymbol{E} \middle | \boldsymbol{U} ,\boldsymbol{V} \right] = \mathbb{E}_{\boldsymbol{D}}\left[\mathbb{E}\left[ \boldsymbol{E} \middle | \boldsymbol{U} ,\boldsymbol{V},\boldsymbol{D} \right] \middle| \boldsymbol{U} ,\boldsymbol{V} \right] = 0$, further implying $\mathbb{E} \left[ \boldsymbol{Y} \middle | \boldsymbol{U}, \boldsymbol{V}, \boldsymbol{D}\right] = \mathbb{E} \left[ \boldsymbol{Y} \middle | \boldsymbol{U}, \boldsymbol{V}\right]$, which is a conditional mean independence. 

\end{remark}

\begin{assumption}\label{assumption: linear span inclusion on latent row factors} (Linear span inclusion on latent row factors). There exists a universal constant $\mu \in \mathbb{N}$, such that given arbitrary $(i,d) \in\mathbb{R}^{m} \times \mathcal{L}$, for all $\mathcal{I}^{(d)}(i) \subseteq [m] \backslash \{i\}$ satisfying $\left|\mathcal{I}^{(d)}(i)\right| \ge \mu$, $u_{i}^{(d)}$ is a linear combination of $u_{l \in \mathcal{I}^{(d)}(i)}^{(d)}$. To rephrase, there exists a $\beta^{(d)}\left(\mathcal{I}^{(d)}(i)\right) \in \mathbb{R}^{\mathcal{I}^{(d)}(i)}$ such that $  u_i^{(d)} = \sum_{l \in \mathcal{I}^{(d)}(i)} \beta_l^{(d)}\left(\mathcal{I}^{(d)}(i)\right) u_{l}^{(d)} . $
  
\end{assumption}


  


The above three assumptions simply extend the basic assumptions of the binary causal matrix completion framework \citep{agarwal2023causal} into multiple treatment setting without considering the connection among different treatment levels. We characterize this low-rank structure across different treatments in the following assumption: 

\begin{assumption}\label{assumption: same latent row factors} (Shared latent row factors). All the treatment levels $d\in \mathcal{L}$ share the same latent row factors, $u_i^{(d)} \equiv u_i$.

\end{assumption}

\textbf{Justification for Assumption \ref{assumption: same latent row factors}}. This assumption is necessary for enabling cross-treatment data integration in our MSNN. It means that while treatment assignment may affect observed outcomes, the underlying latent characteristics associated with each row remain invariant. 

This assumption mirrors Assumption 2 of \citet{agarwal2020synthetic} under a row–column exchange and does not impose stronger structural constraints. For example, when rows correspond to users and columns to items, different treatment levels may represent varying exposure or intervention intensities. In such settings, it is natural to assume that a user’s intrinsic preferences are stable across treatments, while the realized outcomes change through treatment-specific loading matrices. 

We note that this assumption covers the low-rank tensor factorization model $Y_{i j}^{(d)} = \sum_{l\in[r]} u_{il} v_{jl} \lambda_{dl} + \epsilon_{i j}^{(d)}$ mentioned in \citet{agarwal2020synthetic} as a special case, by setting $v_{jl}^{(d)} = v_{jl} \lambda_{dl}$. 

\subsection{Identification}

Below we define $\boldsymbol{A} \coloneqq \mathbb{E}\left[ \boldsymbol{Y} \middle | \boldsymbol{U} ,\boldsymbol{V} \right]$ as the ground truth. The theoretical guarantee of Mixed Synthetic Nearest Neighbors (MSNN) we introduce below is the irrelevance of treatment in the estimation of $\beta$, which is followed by the low-rank assumption on treatment represented by latent row factor (Assumption \ref{assumption: same latent row factors}). 

\begin{lemma}\label{lemma: coefficient beta is irrelevant to treatment}

Under Assumption \ref{assumption: linear span inclusion on latent row factors} and \ref{assumption: same latent row factors}, the index set $\mathcal{I}^{(d)}(i)$ and coefficient $\beta^{(d)}\left(\mathcal{I}^{(d)}(i)\right)$ are irrelevant to treatment $d$, i.e. $\forall d^\prime \in \mathcal{L}$, $u_i^{(d^\prime)} = \sum_{l \in \mathcal{I}^{(d)}(i)} \beta_l^{(d)}\left(\mathcal{I}^{(d)}(i)\right) u_{l}^{(d^\prime)}$ for any valid $\mathcal{I}^{(d)}(i)$ and $\beta_l^{(d)}\left(\mathcal{I}^{(d)}(i)\right)$. 

\end{lemma}

Proof is presented in Appendix \ref{proof: lemma: coefficient beta is irrelevant to treatment}. Due to this lemma, the index set and coefficient are shared across treatments, and thus we omit the index $(d)$ for $\mathcal{I}^{(d)}(i)$ and $\beta^{(d)}$ in the derivation below. 

\begin{theorem}\label{theorem: identification}

Under Assumptions \ref{assumption: low-rank factorization}, \ref{assumption: selection on latent factors}, \ref{assumption: linear span inclusion on latent row factors}, \ref{assumption: same latent row factors}, given $(i,j,d)$ and $\mathcal{I}(i) \in [m] \backslash \{i\}$ such that $D_{lj,l \in \mathcal{I}(i)} = d$, then 

\(
A_{ij}^{(d)} = \sum_{l\in\mathcal{I}(i)} \beta_l\left(\mathcal{I}(i)\right) \mathbb{E}\left[ \tilde{Y}_{lj} \middle | \boldsymbol{U}, \boldsymbol{V}, \boldsymbol{D} \right]. 
\)


\end{theorem}

Proof is presented in Appendix \ref{proof: theorem: identification}. Similar to the identification theorem in \citet{agarwal2023causal}, one can estimate the potential outcome $A_{ij}^{(d)}$ through the linear model parameter $\beta$. The main difference under multiple treatment is that, here the estimation of a valid $\beta$ does not require the data from the same treatment level. 

\section{From SNN to MSNN: Data Integration across Treatment levels}

In this section, we introduce MSNN (Mixed Synthetic Nearest Neighbors) to address the lack of data in multiple treatment setting. 

\subsection{SNN under multiple treatment}

Denote $x(d) \coloneqq \left[ \tilde{Y}_{aj} \right]_{a \in \mathrm{AR}(d)} $, $q(d) \coloneqq \left[ \tilde{Y}_{ib} \right]_{b \in \mathrm{AC}(d)}$. ($q$ is a row vector. ) From the definition of $\mathrm{AR}(d)$, $x(d) = \left[ Y_{aj}^{(d)} \right]_{a \in \mathrm{AR}(d)}$, $q(d) = \left[ Y_{ib}^{(d)} \right]_{b \in \mathrm{AC}(d)}$, which is exactly the potential outcome at treatment level $d$. $(\cdot)^{(k)}$ means the $k^{th}$ group out of total group $K_{\rm SNN}$. Now we present the SNN algorithm in Algorithm \ref{algorithm: SNN}. 

To estimate potential outcome on $(i,j)$ under treatment $d$, trivial SNN on multiple treatment needs to select Anchor Rows (AR) and Anchor Columns (AC) specific to treatment level $d$ such that: $D_{ab: a \in \mathrm{AR}(d), b \in \mathrm{AC}(d)} = d$, $D_{aj: a \in \mathrm{AR}(d)} = d$, $D_{ib: b \in \mathrm{AC}(d)} = d$. 

\begin{algorithm}[t]
  \caption{SNN$(i,j,d)$}
  \label{algorithm: SNN}
  \begin{algorithmic}
    \STATE Input: $\left\{ \lambda^{(k)}(d): k \in [K_{\rm SNN}], d \in \mathcal{L} \right\}$, $\left\{ \left( \mathrm{AR}^{(k)}(d), \mathrm{AC}^{(k)}(d) \right): k \in [K_{\rm SNN}] \right\}$ with mutually disjoint sets $\left\{ \mathrm{AR}^{(k)}(d): k \in [K_{\rm SNN}] \right\}$. 
    \FOR {$k \in [K_{\rm SNN}]$}
      \STATE 1. Define $\boldsymbol{S}^{(k)}(d) = \left[ \tilde{Y}_{ab}: (a,b) \in \mathrm{AR}^{(k)}(d) \times \mathrm{AC}^{(k)}(d) \right]$
      \STATE 2. Compute SVD decomposition $\boldsymbol{S}^{(k)}(d) \leftarrow \sum_{l \ge 1} \hat{\tau}_l^{(k)}(d)  \hat{u}_l^{(k)}(d) \otimes \hat{v}_l^{(k)}(d)$
      \STATE 3. Compute $\hat{\beta}^{(k)}(d) \leftarrow \Big( \sum_{l \le \lambda^{(k)}(d)} \left(1/\hat{\tau}_l^{(k)}(d)\right) \hat{u}_l^{(k)}(d) \otimes \hat{v}_l^{(k)}(d) \Big) q^{(k)}(d)$
      \STATE 4. Compute $\hat{A}_{ij,k}(d) \leftarrow \left\langle x^{(k)}(d), \hat{\beta}^{(k)}(d) \right\rangle$
    \ENDFOR
    \STATE Output $\hat{A}_{ij}^{(d)} \leftarrow  \frac{1}{K_{\rm SNN}} \sum_{k=1}^{K_{\rm SNN}} \hat{A}_{ij,k}(d)$
  \end{algorithmic}
\end{algorithm}

\subsection{Mixed Synthetic Nearest Neighbors: data combination}

MSNN leverages data across different treatment levels as ensured by Theorem \ref{theorem: identification}. The key insight is that, while the linear combination factors $x(d)$ come from the same treatment level as the estimated target treatment level $d$, the estimation of the coefficient vector $\beta$ can incorporate information from other treatments. This observation motivates relaxing the notions of Anchor Rows (AR) and Anchor Columns (AC) to their mixed counterparts - Mixed Anchor Rows (MAR) and Mixed Anchor Columns (MAC) - defined as follows. 

Mixed Anchor Rows (MAR) and Mixed Anchor Columns (MAC) at treatment level $d$ for entry $(i,j)$ satisfy the following conditions: 1. $\forall a \in \mathrm{MAR}(d), b \in \mathrm{MAC}(d)$, $D_{ab: a \in \mathrm{MAR}(d), b \in \mathrm{MAC}(d)} = D_{ib: b \in \mathrm{MAC}(d)} \eqqcolon d(b) \ne 0$ (where we call column treatment level below); 2. $\forall a \in \mathrm{MAR}(d)$, $D_{aj: a \in \mathrm{MAR}(d)} = d$. 

Intuitively, this means that within the block $\mathrm{MAR}(d) \times \{ \mathrm{MAC}(d) \cap \{j\} \}$, every column shares the same (non-zero) treatment assignment as the row $i$ corresponding to those columns. Now $x(d) = \left[ Y_{aj}^{(d)} \right]_{a \in \mathrm{MAR}(d)}$ remains, however $q(d) = \left[ Y_{ib}^{(d(b))} \right]_{b \in \mathrm{MAC}(d)}$ where entries may consist of different treatment levels. 

To balance the probable scale and variance heterogeneity introduced by mixed treatment levels, we also introduce positive weight functions $w(j,d): [n] \times \mathcal{L} \to \mathbb{R}_+$. 

We denote the weighted $q(d)$ as $q_w(d) \coloneqq \left[ w(b,D_{ib}) \cdot \tilde{Y}_{ib} \right]_{b \in \mathrm{MAC}(d)} = \left[ w(b,d(b)) \cdot Y_{ib}^{(d(b))} \right]_{b \in \mathrm{MAC}(d)}$. Now we can introduce our MSNN algorithm in Algorithm \ref{algorithm: MSNN} above. A illustrative comparison between SNN and MSNN is also provided in Figure \ref{fig: Comparison between SNN and MSNN}.    

\begin{algorithm}[t]
  \caption{MSNN$(i,j,d)$}
  \label{algorithm: MSNN}
  \begin{algorithmic}
    \STATE Input: $\left\{ \lambda^{(k)}(d): k \in [K_{\rm MSNN}], d \in \mathcal{L} \right\}$, 
    $\left\{ \left( \mathrm{MAR}^{(k)}(d), \mathrm{MAC}^{(k)}(d) \right): k \in [K_{\rm MSNN}] \right\}$ with mutually disjoint sets $\left\{ \mathrm{MAR}^{(k)}(d): k \in [K_{\rm MSNN}] \right\}$ w.r.t. column treatment levels $d\left( b: b \in \mathrm{MAC}^{(k)}(d) \right)$. 
    \FOR {$k \in [K_{\rm MSNN}]$}
      \STATE 1. Define $\boldsymbol{S}_w^{(k)}(d) = \left[ w\left( b, d( b)\right) \cdot \tilde{Y}_{ab}: (a,b) \in \mathrm{MAR}^{(k)}(d) \times \mathrm{MAC}^{(k)}(d) \right]$
      \STATE 2. Compute SVD decomposition $\boldsymbol{S}_w^{(k)}(d) \leftarrow \sum_{l \ge 1} \hat{\tau}_l^{(k)}(d)  \hat{u}_l^{(k)}(d) \otimes \hat{v}_l^{(k)}(d)$
      \STATE 3. Compute $\hat{\beta}^{(k)}(d) \leftarrow \Big( \sum_{l \le \lambda^{(k)}(d)} \left(1/\hat{\tau}_l^{(k)}(d)\right) \hat{u}_l^{(k)}(d) \otimes \hat{v}_l^{(k)}(d) \Big) q_w^{(k)}(d)$
      \STATE 4. Compute $\hat{A}_{ij,k}(d) \leftarrow \left\langle x^{(k)}(d), \hat{\beta}^{(k)}(d) \right\rangle$
    \ENDFOR
    \STATE Output $\hat{A}_{ij}^{(d)} \leftarrow  \frac{1}{K_{\rm MSNN}} \sum_{k=1}^{K_{\rm MSNN}} \hat{A}_{ij,k}(d)$
  \end{algorithmic}
\end{algorithm}

\begin{figure*}[h]
\centering
\begin{subfigure}{\linewidth}
\includegraphics[width=0.19\linewidth]{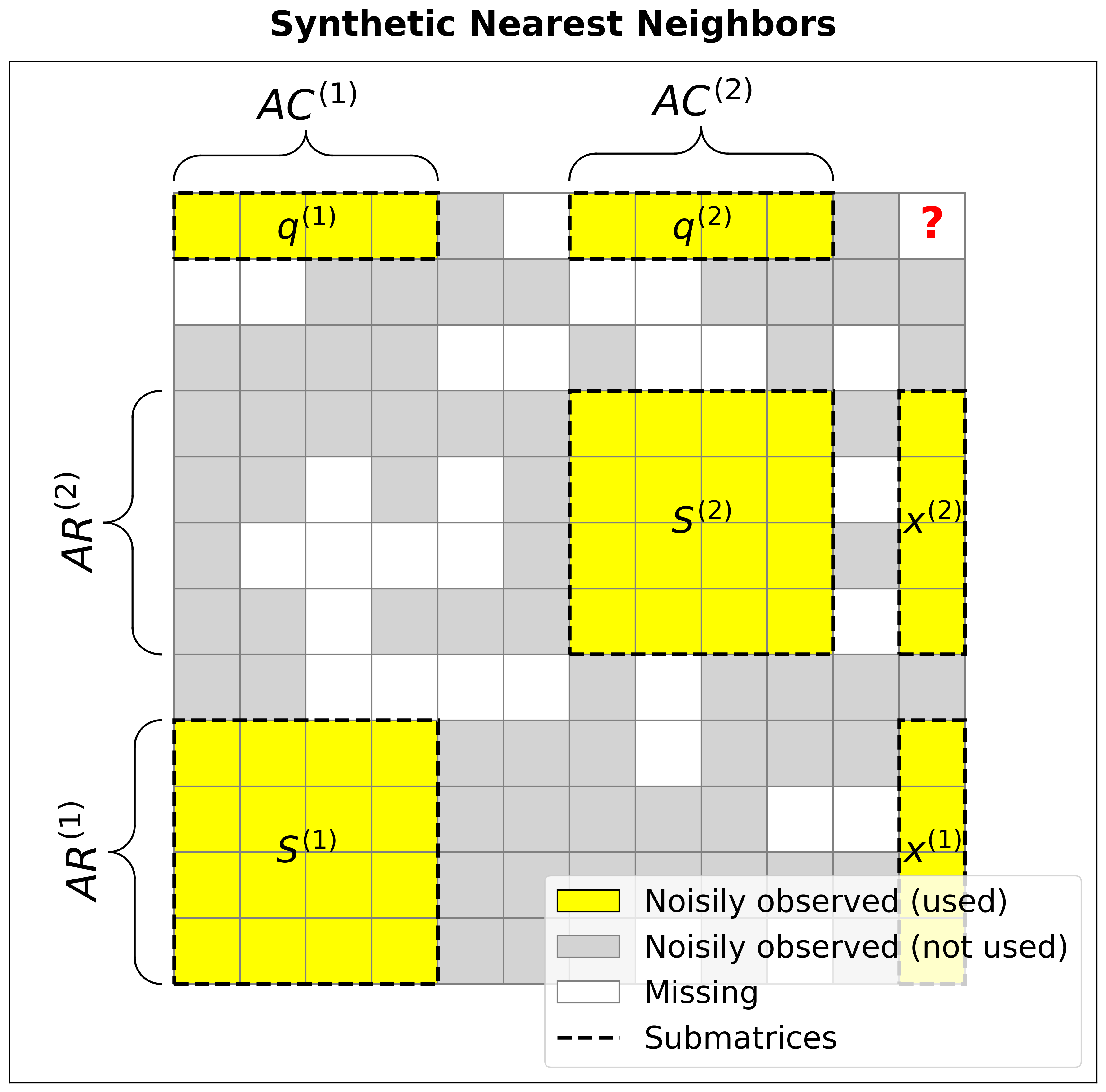}\includegraphics[width=0.19\linewidth]{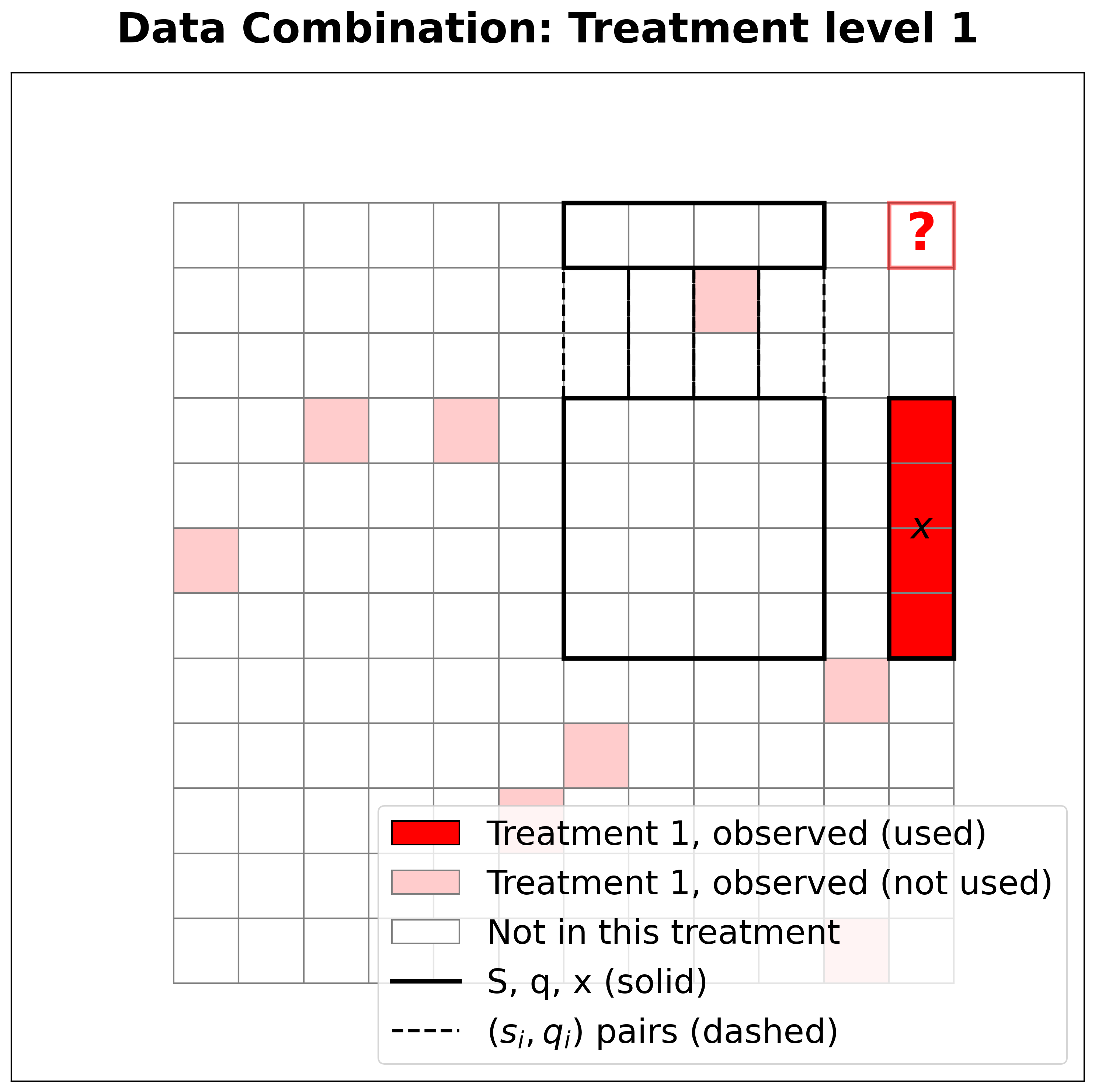} 
    \includegraphics[width=0.19\linewidth]{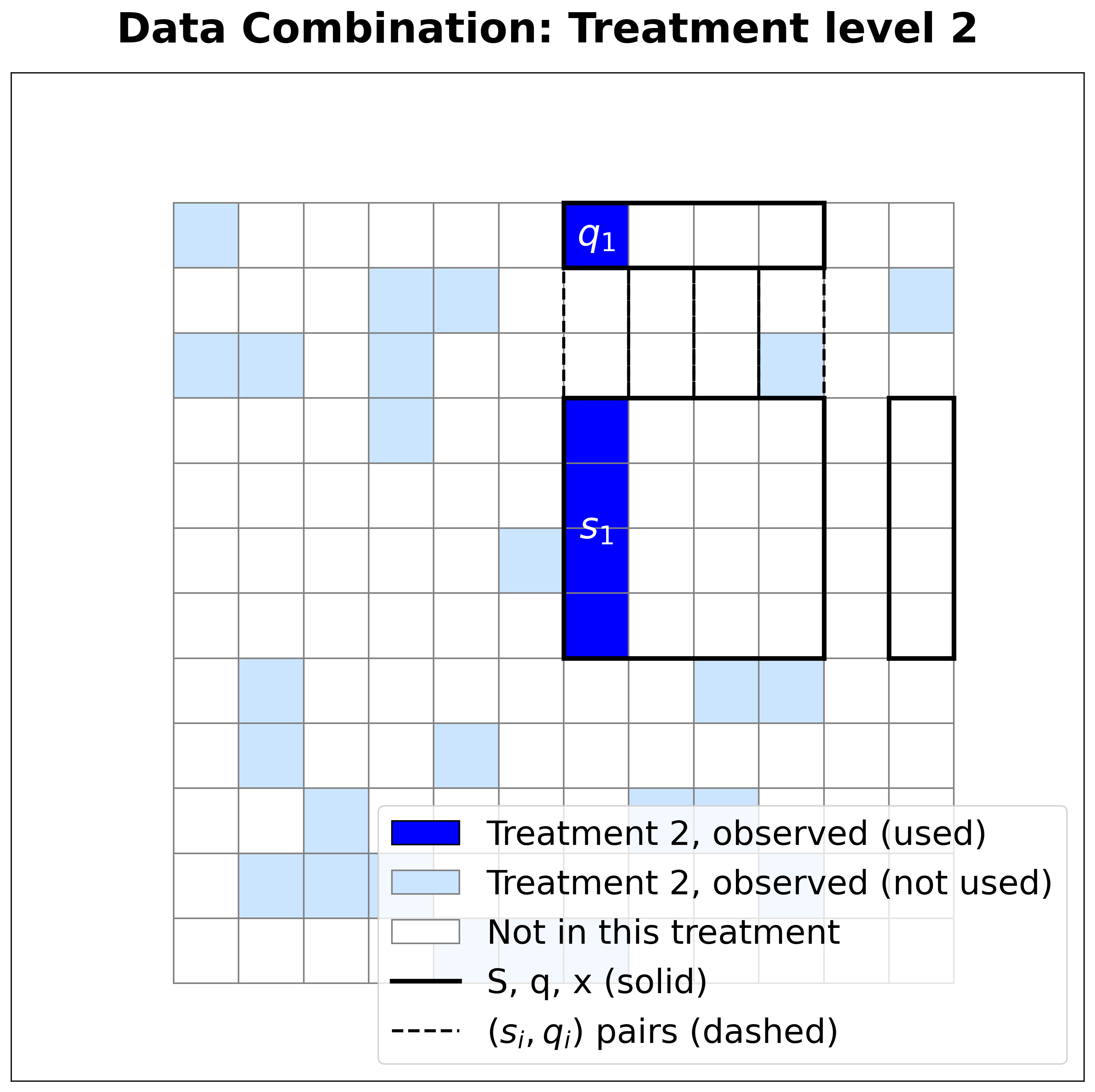} 
    \includegraphics[width=0.19\linewidth]{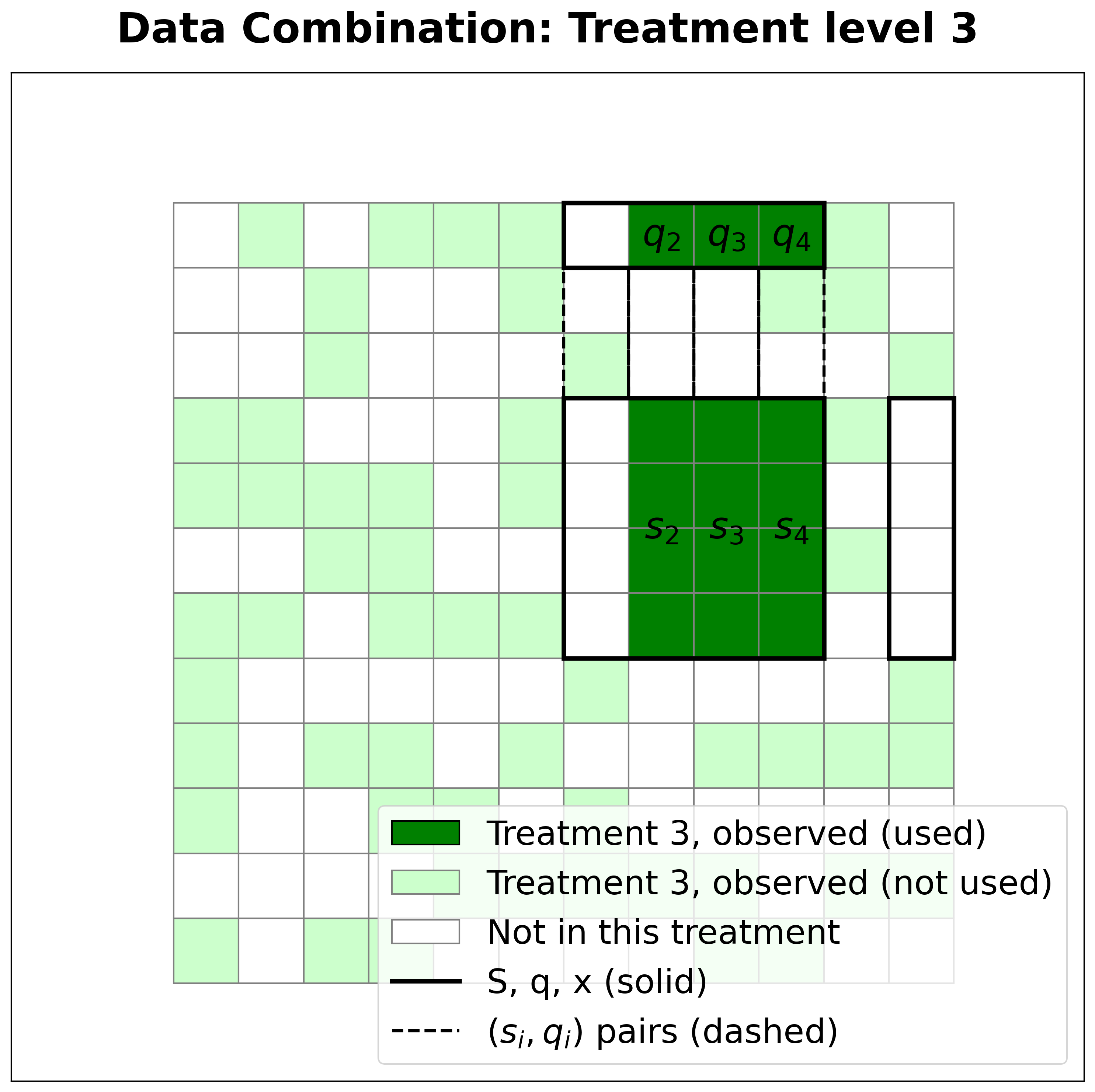} 
    \includegraphics[width=0.19\linewidth]{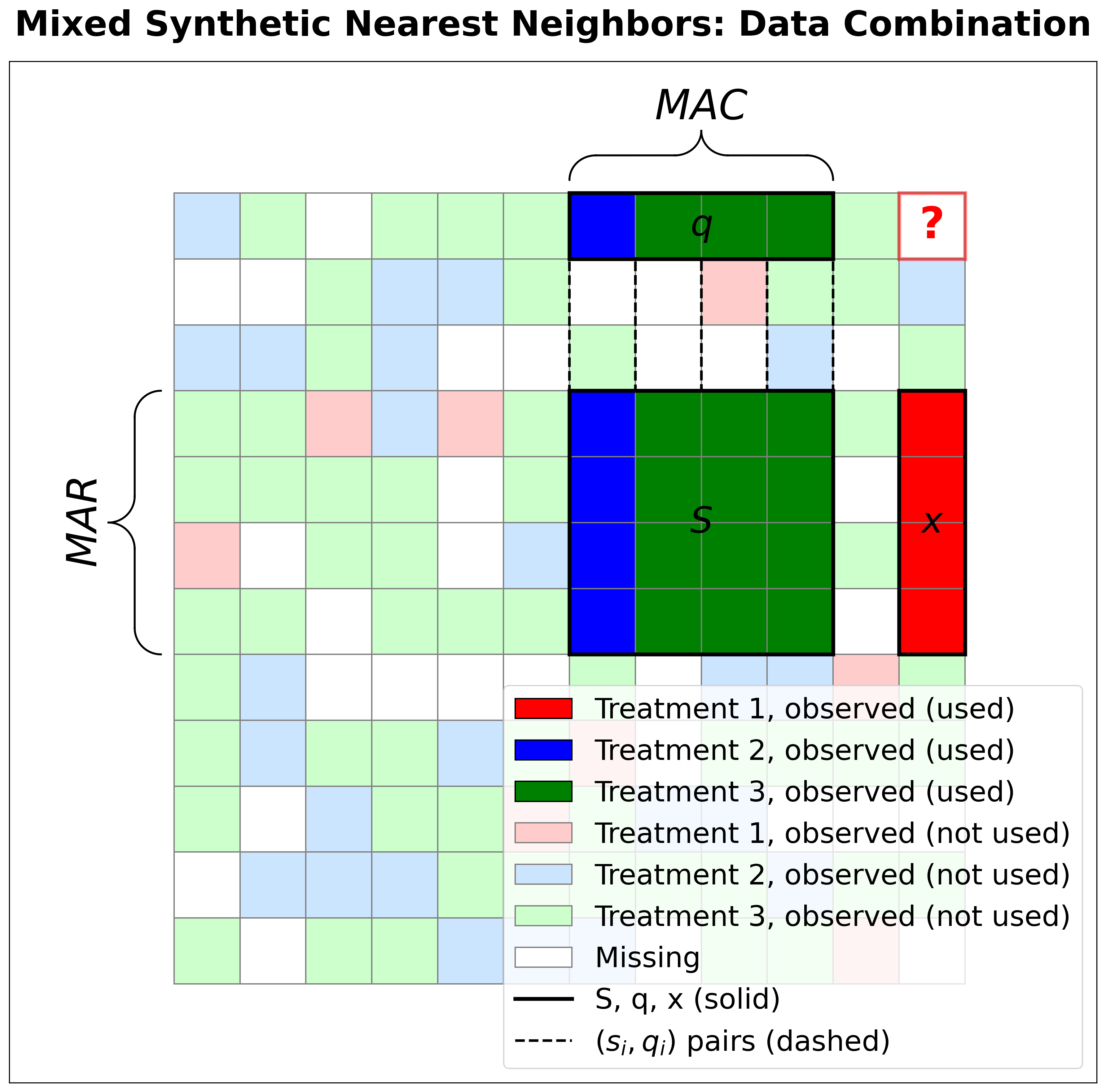} 
\end{subfigure}
\caption{Comparison between SNN and MSNN. The leftmost subfigure illustrates the SNN algorithm with $K_{\rm SNN} = 2$: it requires $\boldsymbol{S}^{(k)}$, $q^{(k)}$ and $x^{(k)}$ are all fully observed at treatment level as the same of the estimated treatment, which is rare under data-scarce levels (e.g. the ``red'' level in the second subfigure). The rest four subfigures explain the procedure of MSNN for a specific subgroup $k$: given entry $(i,j)$ and estimated treatment level (here is ``red''), one need to find a fully observed $x^{(k)}$ under same ``red'' level, but the $\boldsymbol{S}^{(k)}$ and $q^{(k)}$ can be integrated from other treatments (here: ``blue'' and ``green'' level). The only requirement is that for each column of $\boldsymbol{S}^{(k)}$ (namely, $s_i^{(k)}$), its treatments should be as the same as the treatment of corresponding $q_i^{(k)}$, see the third and fourth subfigures. }
\label{fig: Comparison between SNN and MSNN}
\end{figure*}









\begin{remark} Utility of weight $w(b,d(b))$. 
Since $\boldsymbol{S}_w^{(k)}(d)$ is a combination of observed data from different treatment levels, heterogeneity of different level of data may affect its conditional number, causing the matrix to be ill-conditioned and introducing numerical instability to the SVD step of Algorithm \ref{algorithm: MSNN}. We provide a simple illustration below: 

\textbf{Illustration. }Consider $\boldsymbol{S}_w^{(k)}(d)$ consists of two columns of data from treatment level $d(1)$ and $d(2)$, where data from $d(2)$ has large scale of $C$, consider $\boldsymbol{S}_w^{(k)}(d) = \begin{pmatrix}
    1 & 0 \\
    0 & C 
\end{pmatrix}$, then the conditional number of $\boldsymbol{S}_w^{(k)}(d)$ is $C \gg 1$. However if we apply $w(1,d(1))=1$, $w(1,d(1))=1/C$, then $\boldsymbol{S}_w^{(k)}(d) = \begin{pmatrix}
    1 & 0 \\
    0 & 1 
\end{pmatrix}$, which has a conditional number of $1$. 

A natural selection of $w(b,d(b))$ is inversely proportional to the scale of treatment level $d(b)$ (that is, normalizing all observed outcomes to the same scale), under which we establish the following finite-sample bound and asymptotic normality. When the scale factors are unknown in practice, a simple and feasible approach is to estimate using the maximum observed absolute entry within each treatment.

\end{remark}

\subsection{Finding Mixed Anchor Rows and Columns}

Given treatment assignment $\boldsymbol{D}$, we can construct the $\{0,1\}$ valued matrix $\boldsymbol{B}\in \mathbb{R}^{m\times n}$ by $\boldsymbol{B} = \left[ \mathbf{1}\{D_{ab} = D_{ib}, D_{aj} = d, a\ne i, b \ne j\} \right]$, which indicates whether an entry of $\tilde{\boldsymbol{Y}}$ is usable for constructing $\mathrm{MAR}(d)$ and $\mathrm{MAC}(d)$. Giving $K_{\rm MSNN}$ subgroups of $(\mathrm{MAR}^{(k)}(d),\mathrm{MAC}^{(k)}(d))$, they equivalently represents $K_{\rm MSNN}$ all-$1$ submatrices of $\boldsymbol{B}$ with disjoint columns. 

Now if given $|\mathrm{MAR}(d)| = r$ and $|\mathrm{MAC}(d)| = c$, the problem equivalently reduces to selecting $K_{\rm MSNN}$ all-$1$ submatrices from $\boldsymbol{B}$ with disjoint columns. By modifying algorithm $\texttt{AnchorSubMatrix}$ discussed in \cite{agarwal2023causal}, we can construct $\mathrm{MAR}(d)$ and $\mathrm{MAC}(d)$ by our Algorithm \ref{algorithm: MixedAnchorSubMatrix}: $\texttt{MixedAnchorSubMatrix}$. Then given $\mathrm{MAR}(d)$ and $\mathrm{MAC}(d)$, we construct $K_{\rm MSNN}$ subgroups by selecting $\mathrm{MAC}^{(k)}(d) = \mathrm{MAC}(d)$, while partition $\mathrm{MAR}(d)$ into $K_{\rm MSNN}$ subgroups of equal size randomly. 

Here we assume $\texttt{createGraph}$ and $\texttt{maxBiclique}$ as two known algorithms. $\texttt{createGraph}$ takes input $\boldsymbol{B}$ as an bipartite incident matrix for an bipartite graph $\mathcal{G}$ (that is, $\mathcal{G} = (\mathcal{V}_1, \mathcal{V}_2, \mathcal{E})$ where $|\mathcal{V}_1| = m$ and $\mathcal{V}_2 = n$; for $v_i \in \mathcal{V}_1$ and $v_j \in \mathcal{V}_1$, $(v_i, v_j)\in\mathcal{E}$ if and only if $\boldsymbol{B}_{ij}=1$) and outputs $\mathcal{G}$. $\texttt{createGraph}$ takes bipartite $\mathcal{G}$ as input and then outputs the maximal bipartite clique. 

\begin{algorithm}[t]
  \caption{MixedAnchorSubMatrix $(i,j,d)$}
  \label{algorithm: MixedAnchorSubMatrix}
  \begin{algorithmic}
    \STATE Input: \texttt{createGraph}, $\texttt{maxBiclique}$ 
    
    \STATE 1. Assign $\boldsymbol{B} = \left[ \mathbf{1}\{D_{ab} = D_{ib}, D_{aj} = d, a\ne i, b \ne j\} \right]$ 

    \STATE 2. Generate $\mathcal{G} \leftarrow \texttt{createGraph}(\boldsymbol{B})$ 

    \STATE 3. Compute $(\mathcal{V}_1^\prime, \mathcal{V}_2^\prime, \mathcal{E^\prime}) \leftarrow \texttt{maxBiclique}(\mathcal{G})$ 
    
    \STATE 4. Output $\mathrm{MAR}(d) \leftarrow \mathcal{V}_1^\prime$, $\mathrm{MAC}(d) \leftarrow \mathcal{V}_2^\prime$ 
  \end{algorithmic}
\end{algorithm}

\section{Theoretical results}

We further add Assumptions to establish results on finite-sample bound and asymptotic normality. 

\subsection{Additional assumptions}

By following the notation in ~\citet{agarwal2023causal}, below we denote $\mathcal{E} = \{\boldsymbol{U}, \boldsymbol{V}, \boldsymbol{E}\}$. Following the requirements of SNN \citep{agarwal2023causal}, we further states the following assumptions under the multiple treatment setting: 

\begin{assumption}\label{assumption: additional, bounded expected potential outcomes} (Bounded expected potential outcomes with treatment heterogeneity). Conditioned on $\mathcal{E}$, $A_{ij}^{(d)} \in [-f(d), f(d)]$, where $f(d): \mathcal{L} \to \mathbb{R}^{+}$ characterizes the range of potential outcome for each treatment. 

\end{assumption}

\begin{assumption}\label{assumption: additional, subgaussian noise} (Sub-gaussian noise with treatment heterogeneity). Conditioned on $\mathcal{E}$, the error $\epsilon_{ij}^{(d)}$ are independent sub-gaussian mean-zero random variables with $\mathbb{E} \left[ \left( \epsilon_{ij}^{(d)} \right)^2 \right] = \left( \sigma_{ij}^{(d)}\right)^2 $ and $\left\| \epsilon_{ij}^{(d)} \right\|_{\psi_2} \le C \sigma_{ij}^{(d)}$, where $\sigma_{ij}^{(d)} \le f(d) \sigma $ for some universal $\sigma > 0$ and constant $C > 0$. 
  
\end{assumption}

\begin{assumption}\label{assumption: additional, well balanced spectra} (Well balanced spectra under appropriate scaling). For every treatment level $d \in \mathcal{L}$, conditioned on $\mathcal{E}$ and given $(i,j)$ with subgroup $k$, if $w(b,d(b)) = 1/f(d(b))$ ($b\in\mathrm{MAC}^{(k)}(d)$), the non-zero singular values of $\mathbb{E} \left[ \boldsymbol{S}_w^{(k)}(d) \middle | \mathcal{E} \right]$ (namely, $\tau_{l\in\left[ r^{(k)} \right]}^{(k)}(d)$) are well balanced, i.e. there exists universal constants $c, c^\prime > 0$ such that 
$\tau_{r^{(k)}}^{(k)}(d) / \tau_{1}^{(k)}(d) \ge c$ and $\tau_1^{(k)}(d) \le c^\prime \cdot \left| \mathrm{MAR}^{(k)}(d) \right| \left| \mathrm{MAC}^{(k)}(d) \right|$.


\end{assumption}


\begin{assumption}\label{assumption: additional, subspace inclusion}(Subspace inclusion). 
Conditioned on $\mathcal{E}$, given $(i,j,d)$ with subgroup $k$, 

$\mathbb{E} \left[ x^{(k)}(d) \middle | \boldsymbol{U}, \boldsymbol{V}, \boldsymbol{D} \right] \in \col \left( \mathbb{E} \left[ \boldsymbol{S}_w^{(k)}(d) \middle | \boldsymbol{U}, \boldsymbol{V}, \boldsymbol{D} \right] \right). $


\end{assumption}

\subsection{Preservation of Finite-sample bound and Asymptotic Normality}

Under the above assumptions, the data that MSNN leverages (in Algorithm \ref{algorithm: MSNN}) exactly satisfy the conditions that SNN requires, thus the similar conclusions of finite-sample bound and asymptotic normality (refer to Theorem 2 and 3 of \citet{agarwal2023causal}) can be transferred from SNN to MSNN: 

\begin{theorem}\label{theorem: finite-sample bound} (Finite-sample error bound). 
Conditioned on $\mathcal{E}$, for a given entry $(i,j)$ and treatment level $d$, suppose $\left| \mathrm{MAR}^{(k)}(d) \right| \ge \mu$ 
for all subgroup $k\in [K_{\rm MSNN}]$, let Assumptions \ref{assumption: low-rank factorization} to \ref{assumption: additional, well balanced spectra} hold. Further let $K_{\rm MSNN} = o\left(\min_k \left|\mathrm{MAC}^{(k)}\right|^{10} \left|\mathrm{MAR}^{(k)}\right|^{10} \right)$. Finally set $\lambda^{(k)} = \rank\left(\mathbb{E}\left[\boldsymbol{S}_w^{(k)}\right]\right)$, $w(b,d(b)) = 1/f(d(b))$, then $\hat{A}_{ij}^{(d)} - A_{ij}^{(d)} =
   f(d)\cdot O_p\left( \frac{1}{K_{\rm MSNN}} \left\{\text{error}+ \left[ \sum_{k=1}^{K_{\rm MSNN}} \|\tilde{\beta}^{(k)}\|_2^2 \right]^{1/2} \right\} \right), 
$

where $\text{error} \coloneqq \sum_{k=1}^{K_{\rm MSNN}} \left(\text{error}_{k,1} + \text{error}_{k,2}\right)$. Here, $\tilde{\beta}^{(k)} \coloneqq \mathcal{P}_{\mathbb{E}\left[ \boldsymbol{S}_w^{(k)}(d) \middle | \mathcal{E} \right]} \beta^{(k)}$ is the projection of $\beta^{(k)}$ onto the column space of $\mathbb{E}\left[ \boldsymbol{S}_w^{(k)}(d) \middle | \mathcal{E} \right]$, and 

$\text{error}_{k,1} \coloneqq\frac{(r^{(k)})^{1/2}}{|\mathrm{MAC}^{(k)}(d)|^{1/4}}$, 

$\text{error}_{k,2} \coloneqq \frac{(r^{(k)})^{3/2} \|\tilde{\beta}^{(k)}\|_1 \log^{1/2} (|\mathrm{MAC}^{(k)}(d)| |\mathrm{MAR}^{(k)}(d)|)}{\min\{|\mathrm{MAC}^{(k)}(d)|^{1/2}, |\mathrm{MAR}^{(k)}(d)|^{1/2}\}}$. 


\end{theorem}








\begin{theorem}\label{theorem: consistency} (Asymptotic Normality). 
For a given entry $(i,j)$ and treatment level $d$, let the setups of Theorem \ref{theorem: finite-sample bound} holds. Define $(\tilde{\sigma}^{(k)}(d))^2 \coloneqq \sum_{l\in \mathrm{MAR}^{(k)}(d)} (\tilde{\beta}_l^{(k)} \sigma_{lj}^{(d)})^2.$ Further assume that 

(i) $K_{\rm MSNN} \to \infty$, 

(ii) $|\mathrm{MAC^{(k)}(d)}|, |\mathrm{MAR^{(k)}(d)}| \to +\infty$ for each $k$, 

(iii) For each $k$, 
$r^{(k)} \|\tilde{\beta}^{(k)}\|_1^2 \log (|\mathrm{MAC}^{(k)}(d)| |\mathrm{MAR}^{(k)}(d)|) = o\left(\min\{ |\mathrm{MAC}^{(k)}(d)| |\mathrm{MAR}^{(k)}(d)| \}\right)$, 


(iv) $
    f(d) \cdot \text{error} = o\left( \left[ \sum_{k=1}^{K_{\rm MSNN}} (\tilde{\sigma}^{(k)} (d))^2 \right]^{1/2} \right) . 
$

Then conditioned on $\mathcal{E}$, $
  \frac{K_{\rm MSNN}(\hat{A}_{ij}^{(d)} - A_{ij}^{(d)})}{\left[ \sum_{k=1}^{K_{\rm MSNN}} (\tilde{\sigma}^{(k)} (d))^2 \right]^{1/2}} \overset{d}{\to} \mathcal{N}(0,1) . 
$

\end{theorem}

Proof of Theorem \ref{theorem: finite-sample bound} and \ref{theorem: consistency} are provided in Appendix \ref{proof: theorem: finite-sample bound and consistency}. 

\begin{remark} The above two Theorems have the same form with the finite-sample bound and asymptotic normality of SNN, by replacing $|\rm MAR|$, $|\rm MAC|$ with $|\rm AR|$, $|\rm AC|$. 

\end{remark}


\subsection{Sample efficiency for large matrix}\label{subsec: sample efficiency}

To exhibit the efficacy of MSNN, we illustrate that our new Algorithm \ref{algorithm: MSNN} surpasses Algorithm \ref{algorithm: SNN} under the following regime: given the distribution of treatment assignment $\boldsymbol{D}$, we compare the expected number of subgroups we can extract at most from observed outcome $\tilde{\boldsymbol{Y}}$ for MSNN and SNN: $\mathbb{E} [K_{\rm MSNN}]$ and $\mathbb{E} [K_{\rm SNN}]$. We expect that under data-scarce treatment levels, $\mathbb{E} [K_{\rm MSNN}] \gg \mathbb{E} [K_{\rm SNN}]$. 


We analyze $\mathbb{E} \left[K_{(\cdot)}\right]$ mainly because of the following reasons from both theoretical and empirical perspectives: 

\begin{enumerate}
  \item Based on previous section, under the conditions of Theorem \ref{theorem: consistency} (Asymptotic Normality), the error term is of order $K_{(\cdot)}^{-1/2}$. For large matrix size $m,n \gg 1$, this section now presents theoretical guarantee that the estimation error is smaller at the rate of order $\sqrt{\frac{\mathbb{E} [K_{\rm SNN}]}{\mathbb{E} [K_{\rm MSNN}]}}$. 
    
  \item It is computationally hard to find the maximum number of all-$1$ submatrix from a $0-1$ valued matrix (Algorithm \ref{algorithm: MixedAnchorSubMatrix} addresses this by providing $K_{(\cdot)}$ feasible submatrices without optimal guarantee, where $K_{(\cdot)}$ is a tunable hyperparameter), especially when the shape of submatrices are large. Even for Algorithm \ref{algorithm: MixedAnchorSubMatrix}, it is computationally almost unaffordable to find when $|\rm MAR|$, $|\rm MAC|$ exceed $10$. This motivates us to consider another important problem: can we find a feasible Anchor Matrix to estimate under data-scarce treatment levels Subject to the identifiability requirement that the (Mixed) Anchor Rows/Columns are no smaller than a threshold, this probability is linked to the expectation of the subgroup count: 

  $\mathbb{E} \left[ K_{(\cdot)} \right] = \sum_{k \ge 1} \mathbb{P}\left(K_{(\cdot)} \ge k\right)$ 
  
  $ \Longrightarrow \frac{\mathbb{E} \left[ K_{(\cdot)} \right]}{\max K_{(\cdot)}} \le \mathbb{P}\left(K_{(\cdot)} \ge 1\right) \le \mathbb{E} \left[ K_{(\cdot)} \right]$.     

    Generally, the improvement of $\mathbb{E} \left[ K_{(\cdot)} \right]$ leads to the improvement of $\mathbb{P}\left(K_{(\cdot)}\right)$. For the simulation study in the next section, we mainly focus on the data-scarce case that SNN fails to find any Anchor Matrix while MSNN achieves feasible estimation: although $\mathbb{E} \left[ K_{(\cdot)} \right] \ll 1$, we have $\mathbb{E} \left[ K_{(\cdot)} \right] \gtrsim 1$. 
\end{enumerate}






\subsubsection{Case study: missing completely at random}

We consider the simplest case of MCAR. Each entry $(i,j)$ is revealed i.i.d. with the following distribution: 

$\begin{cases}
    \mathbb{P} (D_{ij} = d) = p_d,\,\forall d \in \mathcal{L} , \\
    \mathbb{P} (D_{ij} = 0) = 1 - \sum_{d^\prime \in \mathcal{L}} p_{d^\prime} . 
  \end{cases}$ \label{equation: MCAR}


We require for each treatment level, $p_d > 0$, and $1 - \sum_{d^\prime \in \mathcal{L}} p_{d^\prime} \ge 0$. We denote the largest probability as $p_{\max} \coloneqq \left(\max_{d^\prime} p_{d^\prime} \right)$, and its treatment level $d_{\max} \coloneqq \argmax_d p_d$. For simplicity we denote $\gamma \coloneqq \sum_{d^\prime \in \mathcal{L}} \left( \frac{p_{d^\prime}}{p_{\max}} \right)^{r + 1}$. Now we calculate the expectation of $K_{(\cdot)}$ for both SNN and MSNN in the following theorem: 

\begin{theorem}\label{theorem: expectation of K, MCAR} Expectation of number of samples. 

Consider estimation of any entry $(i,j)$ under treatment level $d$. Suppose the treatment assignment $D$ follows missing completely at random distribution
, for fixed sample size $\mathrm{AR}(d) = \mathrm{MAR}(d) = r$ and $\mathrm{AC}(d) = \mathrm{MAC}(d) = c$, if the probabilities are efficiently small, that is for some $\alpha \in (0,1)$ the following sparsity conditions holds: 

$mr p_d p_{\max}^{\alpha c} = o(1) ,\, nc \gamma p_{\max}^{(1 - \alpha)r + 1 + \alpha} = o(1)$, 


then the expectations of $K_{(\cdot)}$ are bounded by: 

$\frac{\mathbb{E}\left[K_{\rm SNN}(d)\right]}{\binom{m-1}{r}\binom{n-1}{c} p_d^{rc + r + c}} \in \left[ 1 - o(1) , 1 \right]$, 

$\frac{\mathbb{E}\left[K_{\rm MSNN}(d)\right]}{\binom{m-1}{r}\binom{n-1}{c} \gamma^c p_d^{r} p_{\max}^{(r+1)c}} \in \left[ 1 - o(1) , 1 \right] . $


\end{theorem}

Proof is presented in Appendix \ref{proof: theorem: expectation of K, MCAR}. 

\begin{remark} The sparsity condition above 
does not conflicts with $\mathbb{E}[K_{\rm MSNN}] = \omega(1)$. For example, consider 
$r = o(m)$, $c = o(n)$, then $\mathbb{E}\left[K_{\rm MSNN}\right] = \Theta\left( \frac{m^r n^c}{r!c!} \gamma^c p_d^{r} p_{\max}^{(r+1)c} \right) = o\left( \frac{p_{\max}^{- \alpha c}}{r! r^r c! c^c} \right)$, 
which does not conflict with $\mathbb{E}[K_{\rm MSNN}] = \omega(1)$ for fixed $r,c$. 

\end{remark}

Based on the above theorem we analyze the sample efficiency of MSNN quantitatively: 

\begin{corollary}\label{corollary: data efficiency, comparison to SNN, MCAR} 

Estimation efficiency of MSNN compared to trivial SNN. 
Under the setting of Theorem \ref{theorem: expectation of K, MCAR}, 
for estimation of entry $(i,j)$ at treatment level $d$ we have

  $\frac{\mathbb{E}\left[K_{\rm MSNN}(d)\right]}{\mathbb{E}\left[K_{\rm SNN}(d)\right]} = (1 + o(1)) \left[ \sum_{d^\prime \in \mathcal{L}} \left( \frac{p_{d^\prime}}{p_d} \right)^{r + 1} \right]^{c}$ . 


\end{corollary}

\begin{table*}[t]
\centering
\caption{Performance comparison of MSNN and SNN on MCAR data. }
\label{tab: mcar_comparison}
\resizebox{\textwidth}{!}{
\begin{tabular}{l c c c c c c}
\toprule
\multirow{2}{*}{\textbf{Algorithm}} &
\multicolumn{2}{c}{\textbf{Low} ($p(d)=0.01$)} &
\multicolumn{2}{c}{\textbf{Medium} ($p(d)=0.025$)} &
\multicolumn{2}{c}{\textbf{High} ($p(d)=0.05$)} \\
\cmidrule(lr){2-3} \cmidrule(lr){4-5} \cmidrule(lr){6-7}
 & FR \% ($\uparrow$) & MRE ($\downarrow$) & FR \% ($\uparrow$) & MRE ($\downarrow$) & FR \% ($\uparrow$) & MRE ($\downarrow$) \\
\midrule
SNN & $0.03 \pm 0.00$ & $0.806 \pm 0.240$ & $1.20 \pm 0.36$ & $0.577 \pm 0.110$ & $11.34 \pm 0.82$ & $0.515 \pm 0.046$ \\
MSNN (ours) & $\bf 4.69 \pm 1.11$ & $\bf 3.91 \pm 1.09 \;\times 10^{-2}$ & $\bf 63.73 \pm 5.67$ & $\bf 1.18 \pm 0.33 \;\times 10^{-3}$ & $\bf 99.29 \pm 0.81$ & $\bf 7.05 \pm 0.21 \;\times 10^{-4} $ \\
\bottomrule
\end{tabular}
}

\hfill

\centering
\caption{Performance comparison of MSNN and SNN on MNAR data, $\lambda=0.05$. }
\label{tab: mnar_comparison, 0.05}
\resizebox{\textwidth}{!}{
\begin{tabular}{l c c c c c c}
\toprule
\multirow{2}{*}{\textbf{Algorithm}} &
\multicolumn{2}{c}{\textbf{Low} (Proportion $=(1.30 \pm 0.06)\%$)} &
\multicolumn{2}{c}{\textbf{Medium} (Proportion $=(1.49 \pm 0.06)\%$)} &
\multicolumn{2}{c}{\textbf{High} (Proportion $=(2.50 \pm 0.10)\%$)} \\
\cmidrule(lr){2-3} \cmidrule(lr){4-5} \cmidrule(lr){6-7}
 & FR \% ($\uparrow$) & MRE ($\downarrow$) & FR \% ($\uparrow$) & MRE ($\downarrow$) & FR \% ($\uparrow$) & MRE ($\downarrow$) \\
\midrule
SNN & $0.19 \pm 0.07$ & $0.349 \pm 0.139$ & $0.38 \pm 0.12$ & $0.390 \pm 0.143$ & $4.17 \pm 0.84$ & $0.351 \pm 0.050$ \\
MSNN (ours) & $\bf 3.13 \pm 0.41$ & $\bf 0.117 \pm 0.006$ & $\bf 3.26 \pm 0.34$ & $\bf 0.114 \pm 0.007$ & $\bf 4.52 \pm 0.62$ & $\bf 0.106 \pm 0.004$ \\
\bottomrule
\end{tabular}
}

\hfill

\centering
\caption{Performance comparison of MSNN and SNN on MNAR data, $\lambda=0.02$. }
\label{tab: mnar_comparison, 0.02}
\resizebox{\textwidth}{!}{
\begin{tabular}{l c c c c c c}
\toprule
\multirow{2}{*}{\textbf{Algorithm}} &
\multicolumn{2}{c}{\textbf{Low} (Proportion $=(3.54 \pm 0.12)\%$)} &
\multicolumn{2}{c}{\textbf{Medium} (Proportion $=(3.76 \pm 0.12)\%$)} &
\multicolumn{2}{c}{\textbf{High} (Proportion $=(4.59 \pm 0.11)\%$)} \\
\cmidrule(lr){2-3} \cmidrule(lr){4-5} \cmidrule(lr){6-7}
 & FR \% ($\uparrow$) & MRE ($\downarrow$) & FR \% ($\uparrow$) & MRE ($\downarrow$) & FR \% ($\uparrow$) & MRE ($\downarrow$) \\
\midrule
SNN & $9.57 \pm 0.85$ & $0.366 \pm 0.027$ & $11.70 \pm 1.62$ & $0.379 \pm 0.037$ & $22.66 \pm 2.24$ & $0.383 \pm 0.025$ \\
MSNN (ours) & $\bf 26.96 \pm 3.50$ & $\bf 0.129 \pm 0.008$ & $\bf 33.88 \pm 4.20$ & $\bf 0.135 \pm 0.010$ & $\bf 54.16 \pm 4.30$ & $\bf 0.118 \pm 0.009 $ \\
\bottomrule
\end{tabular}
}
\end{table*}

\begin{corollary}\label{corollary: efficiency compared to the treatment level with the most data, MCAR}

Estimation efficiency compared to the treatment level with the most data. 
Under the setting of Theorem \ref{theorem: expectation of K, MCAR}, 
for estimation of entry $(i,j)$ at treatment level $d$ compared the level with most data $d_{\max}$, 

$\frac{\mathbb{E}\left[K_{\rm SNN}(d)\right]}{\mathbb{E}\left[K_{\rm MSNN}(d_{\max})\right]} = O \left( \gamma^{-c} \left( \frac{p_d}{p_{\max}} \right)^{rc+r+c} \right)$, 

$\frac{\mathbb{E}\left[K_{\rm MSNN}(d)\right]}{\mathbb{E}\left[K_{\rm MSNN}(d_{\max})\right]} = O\left( \left( \frac{p_d}{p_{\max}} \right)^{r} \right)$. 




\end{corollary}


Proof are presented in Appendix \ref{proof: corollary: data efficiency, comparison to SNN, MCAR}. 

\begin{remark}

Corollary \ref{corollary: data efficiency, comparison to SNN, MCAR} and \ref{corollary: efficiency compared to the treatment level with the most data, MCAR} together guarantee efficient estimation under data heterogeneity. For a treatment level with sparse data, Corollary \ref{corollary: data efficiency, comparison to SNN, MCAR} shows that the expected number of samples $K_{\rm MSNN}$ increases by a factor of at least the power of $rc$ compared to $K_{\rm SNN}$, which is an exponential improvement, while Corollary \ref{corollary: efficiency compared to the treatment level with the most data, MCAR} shows that the relative efficiency gap between the data-sparsest and the data-richest levels is efficiently reduced - specifically, its dependence on the proportion $p_d / p_{d_{\max}}$ is reduced from a quadratic order $rc$ to a linear order $r$. 

\end{remark}

\section{Experiments}\label{sec: experiments}

This section exhibits the efficacy of MSNN on both synthetic data and real-world data. Experimental details including hyperparamters and real-world experiment background are provided in Appendix \ref{sec: experimental details}. 

\begin{figure*}[h]
\centering
\begin{subfigure}{\linewidth}
    \includegraphics[width=0.33\linewidth]{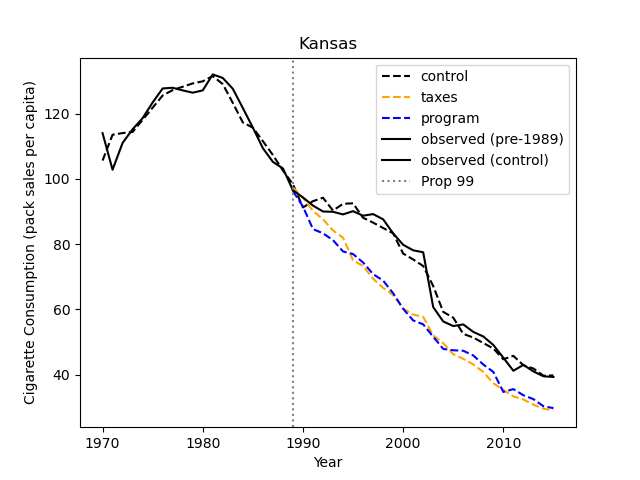} 
    \includegraphics[width=0.33\linewidth]{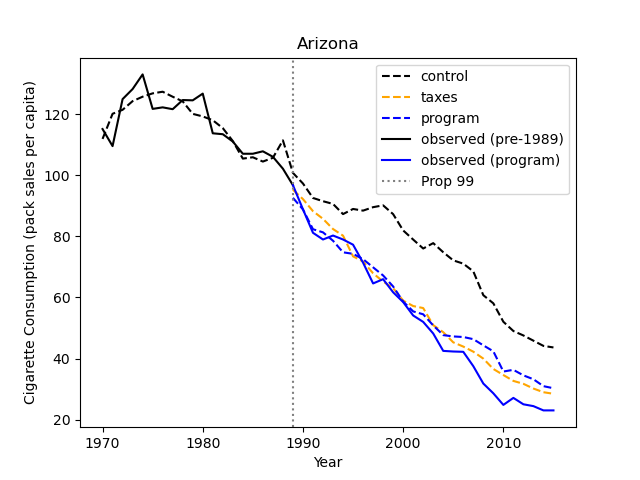} 
    \includegraphics[width=0.33\linewidth]{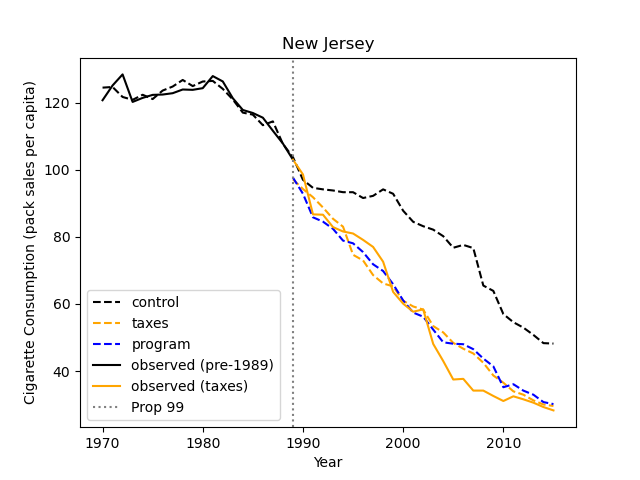} 
\end{subfigure}
\caption{Selected prediction results of MSNN on Proposition 99 study in \citet{abadie2010synthetic}. The three states Kansas, Arizona, New Jersey belong to treatment group of control, program, taxes respectively. The dashed lines are estimation results, while the solid lines indicates real-world observation. Before the year of 1989 (illustrated by the vertical dotted gray lines) all states are in control group so the solid lines are black-colored, after which their color varies. The dotted line indicates the time of Proposition 99 assignment. The dashed lines and solid lines of same color at the same time periods are close to each other, indicating successful validation and thus the correctness of applying our model on this real-world dataset. }
\label{fig: case study}
\end{figure*}

\subsection{Simulation Study}\label{subsec: simulation}

Our theoretical analysis characterizes the expected number of valid anchors. In practice, enumerating multiple anchors is computationally expensive for large matrices; we therefore focus on the practically relevant regime $K_{(\cdot)}=1$, and evaluate the feasibility probability $\mathbb{P}(K_{(\cdot)} \ge 1)$ predicted by our theory along with the prediction error. 

\textbf{Experimental Setup. }Under our multi-treatment framework, we randomly generate latent factors $\boldsymbol{U} \in \mathbb{R}^{m \times r}$ and $\boldsymbol{V}^{(d)} \in \mathbb{R}^{n \times r}$, here $\boldsymbol{U}$ is normalized while $\boldsymbol{V}^{(d)}$ is scaled by a factor $f(d)$. For the treatment assignment, we focus on two settings: (i) MCAR, each entry is revealed at treatment level $d$ with $\mathbb{P}(D_{ij} = d) = p_{\rm MCAR}(d)$; (ii) MNAR, each entry at $(i,j)$ is revealed with $\mathbb{P}(D_{ij} = d) = p_{\rm MNAR}\left(A_{ij}^{d^\prime \in \mathcal{L}} , d \right)$. For (ii) we take the multi-variant softmax function $p_{\rm MNAR}\left(A_{ij}^{d^\prime \in \mathcal{L}} , d \right) = \frac{e^{\lambda A_{ij}^{d}}}{\sum_{d^\prime \ne 0} e^{\lambda A_{ij}^{d^\prime}}}$ ($\lambda$ is a chosen hyper-parameter), generalizing the standard logistic function $\sigma(x) = \frac{e^{x}}{1 + e^{x}}$ discussed by \citet{ma2019missing}. 

Since the data that MSNN can use strictly covers all data that SNN can leverage and thus MSNN outperforms SNN, we mainly focus on cases that SNN fails, that is estimating data-scarce treatment levels. 
Following the implementation of \citet{agarwal2023causal}, we filter out the invalid estimations by testing whether the (Mixed) Anchor Rows/Columns are feasible: approximate satisfaction of $x^{(k)}(d) \in \col(S^{(k)}(d))$ (corresponding to Assumption \ref{assumption: additional, subspace inclusion}) and $q^{(k)}(d) \in \row(S^{(k)}(d))$ (corresponding to Assumption \ref{assumption: linear span inclusion on latent row factors}). We also filter out the estimations which (Mixed) Anchor Matrix is of size $(1,1)$. 
We report both the feasible rate $\text{FR} \coloneqq \frac{\# \text{feasible estimated entries}}{mn}$ and mean relative estimation error $\text{MRE} \coloneqq \frac{\sum_{\hat{A}_{ij}^{(d)} \text{ is feasible}} \left| {(\hat{A}_{ij}^{(d)} - A_{ij}^{(d)})/A_{ij}^{(d)}} \right|}{\# \text{feasible estimated entries}}$ conditioned on feasibility. 


\textbf{Results. }We calculate the feasible ratio (FR) and mean relative error (MRE) of SNN and MSNN for all entries under data-scarce treatment levels ($d=$low, medium, high). Results for MCAR and MNAR settings ($\lambda=0.05$ and $\lambda=0.02$) are reported in Tables \ref{tab: mcar_comparison}, \ref{tab: mnar_comparison, 0.05}, and \ref{tab: mnar_comparison, 0.02}. 

Across all settings, MSNN consistently achieves substantially higher feasible ratios than SNN, often nearly doubling the proportion of entries for which at least one valid anchor is found. Correspondingly, MSNN achieves lower estimation error, with MRE reduced by a factor of two to three for all treatment levels. This improvement arises from two complementary effects: MSNN increases the probability that at least one valid anchor exists, and, conditioned on feasibility, leverages substantially larger mixed anchor row and column sets than SNN. As a result, MSNN uses a larger effective sample size for reconstruction and attains lower estimation error. 

These observations align with the theoretical sample efficiency analysis in Section \ref{subsec: sample efficiency}. While the theory characterizes the expected number of anchors, it predicts a higher probability that $K_{(\cdot)} \ge 1$ under MSNN, which is exactly reflected in the improved feasible ratio observed in practice. When the observed proportion of a treatment level falls below $2.5\%$, the data are too sparse for reliable estimation, and even MSNN can only recover a small fraction of entries. For treatment levels with observed proportions above $4\%$, MSNN operates efficiently with high feasibility and low error, whereas SNN continues to suffer from low feasible ratios and large estimation errors. This highlights that MSNN substantially enlarges the range of sparsity levels under which reliable estimation is possible.


\subsection{Case study: California Proposition 99}\label{subsec: case study}

To exhibit the effect of our method on real-world model, we revisit the classic Proposition 99 study in \citet{abadie2010synthetic}. 
Different from \citet{agarwal2020synthetic} who focuses on time-average treatment effect, we further estimate the counterfactual outcome at \textit{each specific year} for each state. 

In Figure \ref{fig: case study} we select three states covering all treatment groups. In each sub-figure the solid line indicates observed outcomes and dash lines indicate the estimated outcome of our Algorithm. The dash lines which have the same color with solid line at some years indicates validation (for example, the black dash lines before 1989), while others indicate estimated counterfactual outcomes (not observed in reality). The results are consistent with the predicted counterfactual trends until the year of 2000 in the Figure 4 of \citet{agarwal2020synthetic}, though the latter does not discuss guaranteed error bound for each single year in their theoretical framework. 



\section{Conclusions and Future work}

In this work we introduce a new entry-wise causal identification estimator named Mixed Synthetic Nearest Neighbors (MSNN), enabling cross-treatment data integration under the assumption that certain low-rank latent factors are shared across treatments. Theoretically, we prove that MSNN achieves an exponential improvement in sample efficiency for sparse treatment levels under the MCAR setting, while preserving the finite-sample error bound and asymptotic normality guarantees of the original SNN estimator. Experiments on both synthetic and real-world data further demonstrate the efficacy of MSNN under limited data. We hope our work inspires future research in causal matrix completion under multiple treatment levels. 





\newpage

\section*{Impact Statement}
This paper presents work whose goal is to advance the field of Machine
Learning. There are many potential societal consequences of our work, none
which we feel must be specifically highlighted here.


\bibliography{sample}
\bibliographystyle{icml2026}

\newpage
\appendix
\onecolumn

\section{Related Works}




Classical matrix completion assumes data are missing completely at random (MCAR) or at random (MAR). \citep{candes2012exact, keshavan2010matrix} However in many real-world applications, missingness depends on the unobserved entries themselves, i.e., missing not at random (MNAR). Existing work addresses this by jointly modeling the missingness mechanism and
low-rank structure. \citep{ma2019missing, sportisse2020imputation, yang2021tenips} In the causal
inference context, matrix completion can be used to estimate potential outcomes, as in \citet{athey2021matrix,agarwal2023causal,agarwal2023synthetic}. 

Our work studies a general MNAR matrix completion problem under multiple treatments, which is formalized as a three-dimensional tensor completion framework aiming to achieve entry-wise treatment effect identification. 
\citet{agarwal2023causal} addresses the MNAR matrix completion problem under a binary treatment setup, and we show that our result strictly reduces to theirs in the binary case. 

Other approaches in causal matrix completion under multiple treatment levels have different frameworks from  ours. \citet{agarwal2020synthetic} adopts a similar three-dimensional tensor completion framework as ours, but one of the three dimensions is interpreted as time index $t$. Consequently, their goal is to estimate the time-average expected potential outcome during the post-treatment period, rather than entry-wise treatment effects. In contrast, our framework removes this interpretation and focuses on more general matrix completion setting, aiming to construct fine-grained, entry-wise treatment identification under multiple-treatment settings. \citet{gao2025causal} also relies on time sequential data assumption different from our setting. They also utilizes a gradient descent method different from ours. 
\citet{zhang2025generalized} considers a generalized three-dimension causal tensor completion framework under MNAR, where each entry is missing independently, which is different from our assumption that for an element at most one treatment level's potential outcome can be observed (the observation on the treatment dimension are mutually exclusive). \citet{choi2025learningcounterfactualdistributionskernel} proposes KernelNN, estimating the underlying distributions rather than entry-wise value. \citet{dwivedi2024doublyrobustnearestneighbors} proposes a doubly robust nearest neighbors method under binary treatment, differs from our multiple treatment setting. 

\section{Experimental Details}\label{sec: experimental details}

This section further provides experimental details for Section \ref{sec: experiments}. 

\textbf{Hyperparameters of the simulation study.} In Section \ref{subsec: simulation}, the shape of the matrix is $m=300$, $n=100$, rank $r=3$, the relative noise scale $\sigma$ is set to $0.001$. Due to the small scale of experiments, the number of samples $K_{(\cdot)}$ is set to $1$. The treatment levels are $\mathcal{L} = \{\text{low, medium, high, very high}\}$, while the scales of treatment levels are 
$f(d) = \begin{cases}
    1 &, d = \text{low} \\
    5 &, d = \text{medium} \\
    25 &, d = \text{high} \\
    625 &, d = \text{very high} 
\end{cases}$. 

For MCAR simulation (i), the ground truth entries are within $[-f(d), f(d)]$, the probability distribution for each entry is 
$p_{\rm MCAR}(d) = \begin{cases}
    0.115 &, d = 0 \text{ (not observed)} \\
    0.01 &, d = \text{low} \\
    0.025 &, d = \text{medium} \\
    0.05 &, d = \text{high} \\
    0.8 &, d = \text{very high} \\
\end{cases}$
. 
For MNAR simulation (ii), we take the absolute value of each ground truth entry so that they are constrained within $[0, f(d)]$, the hyperparameter $\lambda$ in the probability distribution $p_{\rm MNAR}(d)$ is set to $0.05$ and $0.02$. 

For the feasible rate (FR) and mean relative estimation error (MRE), we report the mean and standard deviation on $10$ repetitions. 

\textbf{Background information of California Proposition 99 Data}

In 1988, California passed the Proposition 99 as the first large-scale anti-tobacco program, followed by other states which introduced similar programs (Arizona, Massachusetts, Oregon, and Florida) or raised tobacco taxes at least 50 cents (Alaska, Hawaii, Maryland, Michigan, New Jersey, New York, Washington). The other 38 states remains status quo, which are treated as the control group. One natural question is that what if a state remains status quo had raised tobacco taxes or introduced anti-tobacco program, or what if a state introduced anti-tobacco program had remained status quo or raised tobacco taxes etc., that is estimating the counterfactual outcomes, as discussed in \citet{agarwal2023synthetic,agarwal2020synthetic}. 

Follow the experiment settings in \citet{agarwal2020synthetic}, we leverage the data from both pre-treatment time and post-treatment time. For simplicity we approximately treat the years after 1989 as post-treatment time. 

\section{Proofs}

\subsection{Proof of Identification}

Proof of Lemma \ref{lemma: coefficient beta is irrelevant to treatment}: 

\begin{proof}\label{proof: lemma: coefficient beta is irrelevant to treatment} The proof is straight forward: 

\begin{equation}
  \begin{aligned}
    u_i^{(d^\prime)} &\overset{A\ref{assumption: same latent row factors}}{=} u_i^{(d)} &&\because \text{Assumption \ref{assumption: same latent row factors}} \\ 
    &\overset{A\ref{assumption: linear span inclusion on latent row factors}}{=} \sum_{l \in \mathcal{I}^{(d)}(i)} \beta_l^{(d)}\left(\mathcal{I}^{(d)}(i)\right) u_{l}^{(d)} &&\because \text{Assumption \ref{assumption: linear span inclusion on latent row factors}} \\ 
    &\overset{A\ref{assumption: same latent row factors}}{=} \sum_{l \in \mathcal{I}^{(d)}(i)} \beta_l^{(d)}\left(\mathcal{I}^{(d)}(i)\right) u_{l}^{(d^\prime)} . &&\because \text{Assumption \ref{assumption: same latent row factors}} 
  \end{aligned}
\end{equation}

\end{proof}

Proof of Theorem \ref{theorem: identification}: 

\begin{proof}\label{proof: theorem: identification}

\begin{equation}
  \begin{aligned}
    A_{ij}^{(d)} &= \mathbb{E}\left[ Y_{ij}^{(d)} \middle | \boldsymbol{U}, \boldsymbol{V} \right] \\ 
    &\overset{A\ref{assumption: low-rank factorization}}{=} \mathbb{E}\left[ \left\langle u_i^{(d)}, v_j^{(d)} \right\rangle + \epsilon_{i j}^{(d)} \middle | \boldsymbol{U}, \boldsymbol{V} \right] &&\because \text{Assumption \ref{assumption: low-rank factorization}} \\ 
    &\overset{A\ref{assumption: selection on latent factors}}{=} \left. \left\langle u_i^{(d)}, v_j^{(d)} \right\rangle \middle | \boldsymbol{U}, \boldsymbol{V} \right. &&\because \text{Assumption \ref{assumption: selection on latent factors}} \\ 
    &= \left. \left\langle u_i^{(d)}, v_j^{(d)} \right\rangle \middle | \boldsymbol{U}, \boldsymbol{V}, \boldsymbol{D} \right. \\ 
    &\overset{L\ref{lemma: coefficient beta is irrelevant to treatment}}{=} \sum_{l \in \mathcal{I}(i)} \left. \beta_l\left(\mathcal{I}(i)\right) \left\langle u_l^{(d)}, v_j^{(d)} \right\rangle \middle | \boldsymbol{U}, \boldsymbol{V}, \boldsymbol{D} \right. &&\because \text{Lemma \ref{lemma: coefficient beta is irrelevant to treatment}} \\ 
    &\overset{A\ref{assumption: selection on latent factors}}{=} \sum_{l \in \mathcal{I}(i)} \beta_l\left(\mathcal{I}(i)\right) \mathbb{E}\left[ \left\langle u_l^{(d)}, v_j^{(d)} \right\rangle + \epsilon_{l j}^{(d)} \middle | \boldsymbol{U}, \boldsymbol{V}, \boldsymbol{D} \right] &&\because \text{Assumption \ref{assumption: selection on latent factors}} \\ 
    &\overset{A\ref{assumption: low-rank factorization}}{=} \sum_{l \in \mathcal{I}(i)} \beta_l\left(\mathcal{I}(i)\right) \mathbb{E}\left[ Y_{lj}^{(d)} \middle | \boldsymbol{U}, \boldsymbol{V}, \boldsymbol{D} \right] &&\because \text{Assumption \ref{assumption: low-rank factorization}} \\ 
    &= \sum_{l \in \mathcal{I}(i)} \beta_l\left(\mathcal{I}(i)\right) \mathbb{E}\left[ \tilde{Y}_{lj} \middle | \boldsymbol{U}, \boldsymbol{V}, \boldsymbol{D} \right] . 
  \end{aligned}
\end{equation}

The last step is from $D_{lj, l \in \mathcal{I}(i)} = d$, which gives $ \tilde{Y}_{lj} =  Y_{lj}^{(d)}$, $\forall l \in \mathcal{I}(i)$. 

\end{proof}

\subsection{Proof of Finite-sample bound and Asymptotic Normality}

Before proceeding on the proofs of Theorem \ref{theorem: finite-sample bound} and \ref{theorem: consistency}, we provide a simple illustration under zero-noise condition: 

\begin{theorem}\label{theorem: further identification}

For Algorithm \ref{algorithm: MSNN}, under Assumptions \ref{assumption: low-rank factorization}, \ref{assumption: same latent row factors}, \ref{assumption: selection on latent factors}, \ref{assumption: linear span inclusion on latent row factors}, \ref{assumption: additional, subspace inclusion}, given $(i,j,d)$ with subgroup $(k)$. If $\left| \mathrm{MAR}^{(k)}(d) \right| \ge \mu$, then 

\begin{equation}
  \begin{aligned}
    A_{ij}^{(d)} &= \mathbb{E} \left[ q_w^{(k)}(d) \middle | \mathcal{E} \right] \mathbb{E}\left[ \boldsymbol{S}_w^{(k)}(d) \middle | \mathcal{E} \right]^{+} \cdot \mathbb{E} \left[ x^{(k)}(d) \middle | \mathcal{E} \right] , 
  \end{aligned}
\end{equation}

where $\cdot^{+}$ denotes pseudo-inverse. 

To rephrase, let $\lambda^{(k)}(d) = \rank\left( \mathbb{E} \left[ \boldsymbol{S}_w^{(k)}(d) \middle | \mathcal{E} \right] \right)$, if $\epsilon_{ij}^{(d)} = 0$ a.s., then $\hat{A}_{ij,k}^{(d)} = A_{ij}^{(d)}$. 

\end{theorem}


\begin{proof}\label{proof: theorem: further identification}

Under Theorem \ref{theorem: identification}, there exists $\beta \in \mathbb{R}^{\mathrm{MAC}^{(k)}(d)}$ s.t. for all $b \in \mathrm{MAC}^{(k)}(d) \cap \{j\}$, $A_{ib}^{(d(b))} = \sum_{l \in \mathrm{MAR}^{(k)}(d)} \beta_l \mathbb{E}\left[ \tilde{Y}_{lj} \middle | \mathcal{E} \right]$. 

Then $\mathbb{E} \left[ q_w^{(k)}(d) \middle | \mathcal{E} \right] = \beta^\top \mathbb{E} \left[ \boldsymbol{S}_w^{(k)}(d) \middle | \mathcal{E} \right]$, $A_{ij}^{(d)} = \beta^\top \mathbb{E} \left[ x^{(k)}(d) \middle | \mathcal{E} \right]$. 

From Assumption \ref{assumption: additional, subspace inclusion}, there exists some $\eta \in \mathbb{R}^{\mathrm{MAC}^{(k)}(d)}$ s.t. 

\begin{equation}\label{equation: linear combination of x}
  \mathbb{E} \left[ x^{(k)}(d) \middle | \mathcal{E} \right] = \mathbb{E} \left[ \boldsymbol{S}_w^{(k)}(d) \middle | \mathcal{E} \right] \eta . 
\end{equation}

Then for any such $\eta$, we have 

\begin{equation}
  A_{ij}^{(d)} = \beta^\top \mathbb{E} \left[ x^{(k)}(d) \middle | \mathcal{E} \right] = \beta^\top \mathbb{E} \left[ \boldsymbol{S}_w^{(k)}(d) \middle | \mathcal{E} \right] \eta = \mathbb{E} \left[ q_w^{(k)}(d) \middle | \mathcal{E} \right] \eta . 
\end{equation}

Let $\eta = \mathbb{E}\left[ \boldsymbol{S}_w^{(k)}(d) \middle | \mathcal{E} \right]^{+} \mathbb{E} \left[ x^{(k)}(d) \middle | \mathcal{E} \right]$, where $\cdot^{+}$ denotes pseudo-inverse, clearly $\mathbb{E}\left[ \boldsymbol{S}_w^{(k)}(d) \middle | \mathcal{E} \right] \eta = \mathcal{P}_{\mathbb{E}\left[ \boldsymbol{S}_w^{(k)}(d) \middle | \mathcal{E} \right]} \mathbb{E} \left[ x^{(k)}(d) \middle | \mathcal{E} \right] = \mathbb{E} \left[ x^{(k)}(d) \middle | \mathcal{E} \right]$, where $\mathcal{P}_{\mathbb{E}\left[ \boldsymbol{S}_w^{(k)}(d) \middle | \mathcal{E} \right]}$ is the projection onto $\col \left( \mathbb{E}\left[ \boldsymbol{S}_w^{(k)}(d) \middle | \mathcal{E} \right] \right)$. Then the proof is completed. 

\end{proof}

Now we prove Theorem \ref{theorem: finite-sample bound} and \ref{theorem: consistency}: 

\begin{proof}\label{proof: theorem: finite-sample bound and consistency}

Suppose we aim to estimate $A_{ij}^{(d)}$ a fixed unit $(i,j)$ and treatment level $d$, and we obtain $K_{\rm MSNN}$ subgroups of $\boldsymbol{S}_w^{(k)}(d)$, $q_w^{(k)}(d)$ and $x^{(k)}(d)$ ($k\in[K_{\rm MSNN}]$). 

Given treatment assignment $\boldsymbol{D}$ with the scaling factor $f(d)$, we define the scaling matrix $\boldsymbol{F} \coloneqq \left[ \frac{1}{f(D_{ij})} \right]_{i\in[m], j\in[n]}$. We further define the normalized latent factor as $\boldsymbol{V}^{(d)\prime} \coloneqq \frac{1}{f(d)} \boldsymbol{V}^{(d)}$, and similarly the normalized error term $\boldsymbol{E}^{(d)\prime} = \frac{1}{f(d)}\boldsymbol{E}^{(d)}$. Since we assume that $\mathbb{E} \left[ \left( \epsilon_{ij}^{(d)} \right)^2 \right] = \left( \sigma_{ij}^{(d)}\right)^2 $ and $\left\| \epsilon_{ij}^{(d)} \right\|_{\psi_2} \le C \sigma_{ij}^{(d)}$, where $\sigma_{ij}^{(d)} \le f(d) \sigma$, then for $\boldsymbol{E}^{\prime}$ we have $\mathbb{E} \left[ \left( \epsilon_{ij}^{(d)\prime} \right)^2 \right] = \left( \sigma_{ij}^{(d)\prime}\right)^2 $ and $\left\| \epsilon_{ij}^{(d)\prime} \right\|_{\psi_2} \le C \sigma_{ij}^{(d)\prime}$, where $\sigma_{ij}^{(d)\prime} \le \sigma$. 

For observed outcome $\tilde{\boldsymbol{Y}}$, we construct $\tilde{\boldsymbol{Y}}^\prime = \tilde{\boldsymbol{Y}} \odot \boldsymbol{F}$, where $\odot$ denotes entry-wise product (Hadamard product). Given $\boldsymbol{S}_w^{(k)}(d)$, $q_w^{(k)}(d)$ and $x^{(k)}(d)$, we denote the submatrices/vectors at the same positions of $\tilde{\boldsymbol{Y}}^\prime$ as $\boldsymbol{S}^{(k)\prime}$, $q^{(k)\prime}$ and $x^{(k)\prime}$ respectively: 

\begin{equation}
  \begin{aligned}
    \boldsymbol{S}^{(k)\prime} &= \left[ \tilde{Y}_{ab}^\prime: (a,b) \in \mathrm{MAR}^{(k)}(d) \times \mathrm{MAC}^{(k)}(d) \right] , \\
    q^{(k)\prime} &= \left[ \tilde{Y}_{ib}^\prime \right]_{b \in \mathrm{MAC}^{(k)}(d)}, \\ 
    x^{(k)\prime} &= \left[ \tilde{Y}_{aj}^\prime \right]_{a \in \mathrm{MAR}^{(k)}(d)} .
  \end{aligned}
\end{equation}

Now from definition, $\boldsymbol{S}_w^{(k)}(d) = \boldsymbol{S}^{(k)\prime}$, $q_w^{(k)}(d) = q^{(k)\prime}$, $x^{(k)}(d) = f(d) \cdot x^{(k)\prime}$. 

Thus if we apply SNN (Algorithm \ref{algorithm: SNN}) to $\boldsymbol{S}^{(k)\prime}$, $q^{(k)\prime}$ and $x^{(k)\prime}$ to obtain another estimation $\hat{A}_{ij,k}^{(d)\prime}$ (though now we do not know its meaning), then the estimation $\hat{A}_{ij,k}^{(d)}$ satisfy: 

\begin{equation}
  \hat{A}_{ij,k}^{(d)} = f(d) \hat{A}_{ij,k}^{(d)\prime} . 
\end{equation}

Now from Theorem 2 and 3 of \citet{agarwal2023causal}, it remains us to prove that $\boldsymbol{S}^{(k)\prime}$, $q^{(k)\prime}$ and $x^{(k)\prime}$ comes from another low-rank factor model. We construct the model as the following procedure: 

\begin{itemize}
    \item As the first step, we select $\boldsymbol{V}^{\prime\prime}$ to be $\boldsymbol{V}^{(d)\prime}$, where $d$ is the estimated treatment level mentioned at the beginning. Then for all $k\in [K_{\rm MSNN}]$, for all $b \in \mathrm{MAC}^{(k)}(d)$, replace the $b^{th}$ column of $\boldsymbol{V}^{\prime\prime}$ by the $b^{th}$ column of $\boldsymbol{V}^{(d(b))\prime}$. (Notice that from the construction of subgroups $k \in [K_{\rm MSNN}]$, $d(b) \coloneqq D_{ab: a \in \mathrm{MAR}^{(k)}(d), b \in \mathrm{MAC}^{(k)}(d)} = D_{ib: b \in \mathrm{MAC}^{(k)}(d)}$. ) 

    \item We conduct similar procedure for the error matrix: first we set $\boldsymbol{E}^{\prime\prime} = \boldsymbol{E}^{(d)\prime}$, then for all $k\in [K_{\rm MSNN}]$ and for all $b \in \mathrm{MAC}^{(k)}(d)$, replace the $b^{th}$ column of $\boldsymbol{E}^{\prime\prime}$ by the $b^{th}$ column of $\boldsymbol{E}^{(d(b))\prime}$. (Now the entries of $\boldsymbol{E}^{\prime\prime}$ are still independent. )

    \item Finally, we set $\boldsymbol{A}^{\prime\prime} = \boldsymbol{U} \boldsymbol{V}^{\prime\prime\top}$, $\tilde{\boldsymbol{Y}}^{\prime\prime} = \boldsymbol{A}^{\prime\prime} + \boldsymbol{E}^{\prime\prime}$. 
\end{itemize}

Now $\tilde{\boldsymbol{Y}}^{\prime\prime} = \boldsymbol{U} \boldsymbol{V}^{\prime\prime\top} + \boldsymbol{E}^{\prime\prime}$ satisfy the assumptions of low-rank factor model discussed in \citet{agarwal2023causal}, while the entries of $\boldsymbol{A}^{\prime\prime}$ are bounded by $[-1,1]$ conditioned on $\mathcal{E}$. From the construction of $\boldsymbol{S}_w^{(k)}(d)$, $q_w^{(k)}(d)$ and $x^{(k)}(d)$, $\tilde{\boldsymbol{Y}}^{\prime\prime}$ and $\tilde{\boldsymbol{Y}}^{\prime}$ are exactly the same on the positions of $\boldsymbol{S}_w^{(k)}(d)$, $q_w^{(k)}(d)$ and $x^{(k)}(d)$. This reduces our proof to the Theorem 2 and 3 of \citet{agarwal2023causal} and thus completes the proof of finite-sample bound and asymptotic normality. 

\end{proof}

\subsection{Proof of Data Efficiency}

For simplicity, below we denote the $K_{(\cdot)}^{'}(d)$ to be the number of subgroups with the constraint of disjoint $x^{(k)}(d)$ removed, $K_{(\cdot)}^{p}(d)$ to be the number of pairs of subgroups $k_1$, $k_2$ that $x^{(k_1)} \cap x^{(k_1)} \ne \emptyset$. 

\subsubsection{Data Efficiency Under Missing Completely At Random}

At the beginning, we need to bound the expectation of the maximum number of samples $\mathbb{E}[K_{(\cdot)}]$ by the following lemma: 

\begin{lemma}\label{lemma: expectation of K, upper and lower bounds} 

Given any entry $(i,j)$ under treatment $d$, for $|\mathrm{AR}(d)| = |\mathrm{MAR}(d)| = r$, $|\mathrm{AC}(d)| = |\mathrm{MAC}(d)| = c$, the expectation of its $K_{\rm SNN}$ and $K_{\rm MSNN}$ can be upper/lower bounded by: 

\begin{equation}
  \begin{aligned}
    &\mathbb{E} \left[ K_{(\cdot)}(d) \right] / \mathbb{E} \left[ K_{(\cdot)}^\prime(d) \right] \\ \in& \left[ \left( 1 + \frac{2 \mathbb{E} \left[ K_{(\cdot)}^{p}(d) \right]}{\mathbb{E} \left[ K_{(\cdot)}^\prime(d) \right]}\right)^{-1}  ,\, 1 \right] . 
  \end{aligned}
\end{equation}

\end{lemma}



\begin{proof}\label{proof: lemma: expectation of K, upper and lower bounds}

For simplicity, we omit $(d)$ in $K_{(\cdot)}(d)$, $K_{(\cdot)}^\prime(d)$ and $K_{(\cdot)}^p(d)$ during the proof. 

The upper bound is constructed by omitting the row-disjoint constraint: $K_{(\cdot)} \le K_{(\cdot)}^\prime$. By taking expectation for both sides the upper bound is proven. 

For the lower bound, we formalize the original problem into a graph problem with randomness, then apply the Caro-Wei Theorem \citep{caro1979new,wei1981lower}. 

Denote the index set of all possible $\mathrm{(M)AR}(d)\times \mathrm{(M)AC}(d)$ as $\mathcal{S}$ ($|\mathcal{S}| = \binom{m-1}{r}\binom{n-1}{c}$). Define the edge set $E_{\mathcal{S}} \coloneqq \left\{(k_1,k_2): k_{1,2}\in \mathcal{S}, \mathrm{(M)AR}^{(k_1)}(d) \cap \mathrm{(M)AR}^{(k_2)}(d) \ne \emptyset \right\}$, the base graph $\mathcal{H} \coloneqq (\mathcal{S}, E_{\mathcal{S}})$. 

To address the randomness of activation, for index $k\in\mathcal{S}$ define $X_k = \mathbf{1}\{ \mathrm{(M)AR}^{(k)}(d)\times \mathrm{(M)AC}^{(k)}(d) \text{ is valid} \}$ to be its indicator function. Specifically, for SNN $X_k = \mathbf{1}\{ D_{ab: a\in \mathrm{AR}^{(k)}(d), b\in \mathrm{AC}^{(k)}(d)} = D_{aj: a\in \mathrm{AR}^{(k)}(d)} = D_{ib: b\in \mathrm{AC}^{(k)}(d)} = d \}$, for MSNN $X_k = \mathbf{1}\{ D_{ab: a\in \mathrm{MAR}^{(k)}(d), b\in \mathrm{MAC}^{(k)}(d)} = D_{ib: b\in \mathrm{MAC}^{(k)}(d)}, D_{aj: a\in \mathrm{MAR}^{(k)}(d)} = d \}$. Thus $X_k = 1$ means subgroup $k$ can be a valid sample. 

Under a given treatment assignment $\boldsymbol{D}$, we define the induced subgraph $\mathcal{G}(\boldsymbol{D})$ of $\{\mathcal{S}, E_{\mathcal{S}}\}$ on the active vertex set $V(\boldsymbol{D}) \coloneqq \{ k \in \mathcal{S}: X_k = 1 \}$. Its edge set is naturally $E_{\mathcal{S}}(\boldsymbol{D}) = E_{\mathcal{S}} \cap (V(\boldsymbol{D}) \times V(\boldsymbol{D}))$. Then $K_{(\cdot)}$ is the independence number (size of the largest independent set) of the induced subgraph $\mathcal{G} = \left( V(\boldsymbol{D}), E_{\mathcal{S}}(\boldsymbol{D}) \right)$.  

From the Caro-Wei Theorem, 

\begin{equation}
  K_{(\cdot)} \mid \boldsymbol{D} \ge \sum_{k_1\in V(\boldsymbol{D})} \frac{1}{1 + \deg(k_1)} = \sum_{k_1\in V(\boldsymbol{D})} \frac{1}{1 + \sum_{k_2\in V(\boldsymbol{D}) \backslash \{k_1\} }\boldsymbol{1}\{(k_1, k_2)\in E_{\mathcal{S}}(\boldsymbol{D})\}} , 
\end{equation}

where $\deg(k_1) \coloneqq \sum_{k_2\in V(\boldsymbol{D}) \backslash \{k_1\} }\boldsymbol{1}\{(k_1, k_2)\in E_{\mathcal{S}}(\boldsymbol{D})\}$ is the degree of vertex $k_1$. 

By taking expectation over $\boldsymbol{D}$, 

\begin{equation}
  \begin{aligned}
    \mathbb{E}\left[ K_{(\cdot)} \right] &\ge \mathbb{E}\left[ \sum_{k_1\in V(\boldsymbol{D})} \frac{1}{1 + \sum_{k_2\in V(\boldsymbol{D}) \backslash \{k_1\}}\boldsymbol{1}\{(k_1, k_2)\in E_{\mathcal{S}}(\boldsymbol{D})\}} \right] \\ 
    &= \mathbb{E}\left[ \sum_{k_1\in \mathcal{S}} \frac{X_{k_1}}{1 + \sum_{k_2\in \mathcal{S} \backslash \{k_1\}}\boldsymbol{1}\{(k_1, k_2)\in E_{\mathcal{S}}\} X_{k_1} X_{k_2}} \right] \\ 
    &= \sum_{k_1\in \mathcal{S}} \mathbb{E}\left[ \frac{X_{k_1}}{1 + \sum_{k_2\in \mathcal{S} \backslash \{k_1\}}\boldsymbol{1}\{(k_1, k_2)\in E_{\mathcal{S}}\} X_{k_1} X_{k_2}} \right] \\ 
    &= \sum_{k_1\in \mathcal{S}} \mathbb{P}(X_{k_1}=1) \cdot \mathbb{E}\left[ \frac{1}{1 + \sum_{k_2\in \mathcal{S} \backslash \{k_1\}}\boldsymbol{1}\{(k_1, k_2)\in E_{\mathcal{S}}\} X_{k_2}} \middle| X_{k_1}=1 \right] . 
  \end{aligned}
\end{equation}

By applying Jensen's inequality, 

\begin{equation}
  \begin{aligned}
    \mathbb{E}\left[ K_{(\cdot)} \right]
    &\ge \sum_{k_1\in \mathcal{S}} \frac{\mathbb{P}(X_{k_1}=1)}{\mathbb{E}\left[ 1 + \sum_{k_2\in \mathcal{S} \backslash \{k_1\}}\boldsymbol{1}\{(k_1, k_2)\in E_{\mathcal{S}}\} X_{k_2} \middle| X_{k_1}=1 \right]} \\ 
    &= \sum_{k_1\in \mathcal{S}} \frac{\mathbb{P}(X_{k_1}=1)}{1 + \sum_{k_2\in \mathcal{S} \backslash \{k_1\}} \boldsymbol{1}\{(k_1, k_2)\in E_{\mathcal{S}}\} \mathbb{P}\left[ X_{k_2}=1 \middle| X_{k_1}=1 \right]} \\ 
    &= \sum_{k_1\in \mathcal{S}} \frac{\mathbb{P}(X_{k_1})^2}{\mathbb{P}(X_{k_1}) + \sum_{\substack{k_2\in \mathcal{S} \backslash \{k_1\}, \\ (k_1, k_2)\in E_{\mathcal{S}}}} \mathbb{P}\left( X_{k_1} \cap X_{k_2} \right)} . \\ 
  \end{aligned}
\end{equation}

By Titu's Lemma (or the Engel's Form of Cauchy-Schwarz inequality), for positive $a_k, b_k$, $\sum_{k}\frac{a_k^2}{b_k} \ge \frac{(\sum_k a_k)^2}{\sum_k b_k}$. This gives 

\begin{equation}
  \begin{aligned}
    \mathbb{E}\left[ K_{(\cdot)} \right]
    &\ge \frac{\left(\sum_{k_1\in \mathcal{S}} \mathbb{P}(X_{k_1})\right)^2}{\sum_{k_1\in \mathcal{S}}\mathbb{P}(X_{k_1}) + \sum_{\substack{k_{1,2}\in \mathcal{S}, \\ (k_1, k_2)\in E_{\mathcal{S}}}} \mathbb{P}\left( X_{k_1} \cap X_{k_2} \right)} = \frac{\mathbb{E}\left[K_{(\cdot)}^\prime \right]^2}{\mathbb{E}\left[K_{(\cdot)}^\prime \right] + 2\mathbb{E}\left[K_{(\cdot)}^p \right]} . \\ 
  \end{aligned}
\end{equation}

The last step is from the definition of $K_{(\cdot)}^p$. 
Then the proof for both sides are completed. 

\end{proof}

To prove Theorem \ref{theorem: expectation of K, MCAR}, we first provide a Lemma on Summation of combinatorial numbers. 

\begin{lemma}\label{lemma: summation of combinatorial numbers, MCAR} Summation of combinatorial numbers. 

For $p_1, p_2, p_3 > 0$, we have

\begin{equation}
  \begin{aligned}
    &\sum_{\substack{0\le r^\prime \le r-1, \\ 0\le c^\prime \le c, \\ (r^\prime, c^\prime) \ne (0,0)}} \binom{M}{r^\prime} \binom{r}{r^\prime} \binom{N}{c^\prime} \binom{c}{c^\prime} p_1^{r^\prime} p_2^{(r+1)c^\prime} p_3^{r^\prime(c-c^\prime)} \\ 
    \le& \left(I_0\left(2\sqrt{Mr p_1 p_3^{c}}\right) - 1\right) + \left(I_0\left(2\sqrt{Nc p_2^{r+1}}\right) - 1\right) \\+& \left(I_0\left(2\sqrt{Mr p_1 p_3^{\alpha c}}\right) - 1\right) \left(I_0\left(2\sqrt{Nc p_2^{r+1} p_3^{\alpha (1-r)}}\right) - 1\right) . 
  \end{aligned}
\end{equation}

where $\alpha \in (0,1)$ is an arbitrary constant, $I_0(x) \coloneqq \sum_{k=0}^{+\infty} \frac{(x/2)^{2k}}{(k!)^2}$ is the Modified Bessel Function of the First Kind, Order Zero. 
    
\end{lemma}

\begin{proof}

By reorganizing the summation, 

\begin{equation}
  \begin{aligned}
    &\sum_{\substack{0\le r^\prime \le r-1, \\ 0\le c^\prime \le c, \\ (r^\prime, c^\prime) \ne (0,0)}} \binom{M}{r^\prime} \binom{r}{r^\prime} \binom{N}{c^\prime} \binom{c}{c^\prime} p_1^{r^\prime} p_2^{(r+1)c^\prime} p_3^{r^\prime(c-c^\prime)} \\ 
    =& \sum_{1\le r^\prime \le r-1} \binom{M}{r^\prime} \binom{r}{r^\prime} (p_1 p_3^{c})^{r^\prime} + \sum_{1\le c^\prime \le c} \binom{N}{c^\prime} \binom{c}{c^\prime} p_2^{(r+1)c^\prime} \\ 
    +& \sum_{\substack{1\le r^\prime \le r-1, \\ 1\le c^\prime \le c}} \binom{M}{r^\prime} \binom{r}{r^\prime} \binom{N}{c^\prime} \binom{c}{c^\prime} p_1^{r^\prime} p_2^{(r+1)c^\prime} p_3^{r^\prime(c-c^\prime)} . \\ 
  \end{aligned}
\end{equation}

By applying $\binom{x}{k} \le \frac{x^k}{k!}$, 

\begin{equation}
  \begin{aligned}
    \sum_{r^\prime = 1}^{r-1} \binom{M}{r^\prime} \binom{r}{r^\prime} (p_1 p_3^{c})^{r^\prime} 
    &\le \sum_{r^\prime = 1}^{r-1} \frac{(Mr p_1 p_3^{c})^{r^\prime}}{(r^\prime!)^2} 
    \le \sum_{r^\prime = 1}^{+\infty} \frac{(Mr p_1 p_3^{c})^{r^\prime}}{(r^\prime!)^2} 
    \le I_0\left(2\sqrt{Mr p_1 p_3^{c}}\right) - 1 \\
    \sum_{c^\prime = 1}^{c} \binom{N}{c^\prime} \binom{c}{c^\prime} p_2^{(r+1)c^\prime} 
    &\le \sum_{c^\prime = 1}^{c} \frac{(Nc p_2^{r+1})^{c^\prime}}{(c^\prime!)^2} 
    \le \sum_{c^\prime = 1}^{+\infty} \frac{(Nc p_2^{r+1})^{c^\prime}}{(c^\prime!)^2} 
    \le I_0\left(2\sqrt{Nc p_2^{r+1}}\right) - 1 . \\
  \end{aligned}
\end{equation}

For the cross term, by appropriate scaling, 

\begin{equation}
  \begin{aligned}
    &\sum_{\substack{1\le r^\prime \le r-1, \\ 1\le c^\prime \le c}} \binom{M}{r^\prime} \binom{r}{r^\prime} \binom{N}{c^\prime} \binom{c}{c^\prime} p_1^{r^\prime} p_2^{(r+1)c^\prime} p_3^{r^\prime(c-c^\prime)} \\ 
    =& \sum_{\substack{1\le r^\prime \le r-1, \\ 1\le c^\prime \le c}} \binom{M}{r^\prime} \binom{r}{r^\prime} \binom{N}{c^\prime} \binom{c}{c^\prime} \left( p_1 p_3^{\alpha c} \right)^{r^\prime} \left( p_2^{r+1} p_3^{\alpha (1-r)} \right)^{c^\prime} p_3^{\alpha(r-r^\prime-1)c^\prime + (1-\alpha)r^\prime (c-c^\prime)} \\ 
    \le& \sum_{\substack{1\le r^\prime \le r-1, \\ 1\le c^\prime \le c}} \binom{M}{r^\prime} \binom{r}{r^\prime} \binom{N}{c^\prime} \binom{c}{c^\prime} \left( p_1 p_3^{\alpha c} \right)^{r^\prime} \left( p_2^{r+1} p_3^{\alpha (1-r)} \right)^{c^\prime} \\ 
    =& \left[ \sum_{r^\prime=1}^{r-1} \binom{M}{r^\prime} \binom{r}{r^\prime} \left( p_1 p_3^{\alpha c} \right)^{r^\prime} \right] \left[ \sum_{c^\prime=1}^{c} \binom{N}{c^\prime} \binom{c}{c^\prime} \left( p_2^{r+1} p_3^{\alpha (1-r)} \right)^{c^\prime} \right] \\ 
    \le& \left(I_0\left(2\sqrt{Mr p_1 p_3^{\alpha c}}\right) - 1\right) \left(I_0\left(2\sqrt{Nc p_2^{r+1} p_3^{\alpha (1-r)}}\right) - 1\right) . \\
  \end{aligned}
\end{equation}

By combining result above, the proof is completed. 

\end{proof}

Now we present the proof of Theorem \ref{theorem: expectation of K, MCAR} under a relaxed sparsity condition: 

\begin{equation}\label{equation: sparse, relaxed, MCAR}
  mr p_d p_{\max}^{\alpha c} \le 1 ,\, nc \gamma p_{\max}^{(1 - \alpha)r + 1 + \alpha} \le 1 . 
\end{equation}

\begin{proof}\label{proof: theorem: expectation of K, MCAR}

For simplicity, we omit $(d)$ in $K_{(\cdot)}(d)$, $K_{(\cdot)}^\prime(d)$ and $K_{(\cdot)}^p(d)$ during the proof. 

For the upper bound, we consider $\mathbb{E}\left[K_{(\cdot)}^\prime \right]$:  

\begin{equation}
  \begin{aligned}
    &\mathbb{E}\left[K_{\rm SNN}\right] 
    \le \mathbb{E}\left[K_{\rm SNN}^\prime\right] \\ 
    =& \mathbb{E} \left[ \sum_{\substack{\mathrm{AR}(d) \subseteq [m]\backslash{i}, \\ \mathrm{AC}(d) \subseteq [n]\backslash{j}}} \mathbf{1}\{D_{ab, a \in \mathrm{AR}(d), b \in \mathrm{AC}(d)} = D_{aj, a \in \mathrm{AR}(d)} = D_{ib, b \in \mathrm{AC}(d)} = d\} \right] \\
    =& \sum_{\substack{\mathrm{AR}(d) \subseteq [m]\backslash{i}, \\ \mathrm{AC}(d) \subseteq [n]\backslash{j}}} \mathbb{E} \left[ \mathbf{1}\{D_{ab, a \in \mathrm{AR}(d), b \in \mathrm{AC}(d)} = D_{aj, a \in \mathrm{AR}(d)} = D_{ib, b \in \mathrm{AC}(d)} = d\} \right] \\
    =& \sum_{\substack{\mathrm{AR}(d) \subseteq [m]\backslash{i}, \\ \mathrm{AC}(d) \subseteq [n]\backslash{j}}} p_d^{rc + r + c} = \binom{m-1}{r}\binom{n-1}{c} p_d^{rc + r + c} , 
  \end{aligned}
\end{equation}

and 

\begin{equation}
  \begin{aligned}
    &\mathbb{E}\left[K_{\rm MSNN}\right] 
    \le \mathbb{E}\left[K_{\rm MSNN}^\prime\right] \\ 
    =& \mathbb{E} \left[ \sum_{\substack{\mathrm{MAR}(d) \subseteq [m]\backslash{i}, \\ \mathrm{MAC}(d) \subseteq [n]\backslash{j}}} \mathbf{1}\{D_{ab, a \in \mathrm{MAR}(d), b \in \mathrm{MAC}(d)} = D_{ib, b \in \mathrm{MAC}(d)}, D_{aj, a \in \mathrm{MAR}(d)} = d\} \right] \\
    =& \sum_{\substack{\mathrm{MAR}(d) \subseteq [m]\backslash{i}, \\ \mathrm{MAC}(d) \subseteq [n]\backslash{j}}} \mathbb{E} \left[ \mathbf{1}\{D_{ab, a \in \mathrm{MAR}(d), b \in \mathrm{MAC}(d)} = D_{ib, b \in \mathrm{MAC}(d)}, D_{aj, a \in \mathrm{MAR}(d)} = d\} \right] \\
    =& \sum_{\substack{\mathrm{MAR}(d) \subseteq [m]\backslash{i}, \\ \mathrm{MAC}(d) \subseteq [n]\backslash{j}}} p_d^{r} \cdot \mathbb{P}\left( D_{ab, a \in \mathrm{MAR}(d), b \in \mathrm{MAC}(d)} = D_{ib, b \in \mathrm{MAC}(d)} \right) \\
    =& \sum_{\substack{\mathrm{MAR}(d) \subseteq [m]\backslash{i}, \\ \mathrm{MAC}(d) \subseteq [n]\backslash{j}}} \left[ p_d^{r} \cdot \left( \sum_{d^\prime \in \mathcal{L}} p_{d^\prime}^{r + 1}\right )^{c} \right] = \binom{m-1}{r}\binom{n-1}{c} \gamma^c p_d^{r} p_{\max}^{(r+1)c} . 
  \end{aligned}
\end{equation}

Then we upper bound the number of overlapped pairs $\mathbb{E}\left[K_{(\cdot)}^p \right]$.  

We denote events $I_{\mathrm{SNN}}^{(k)} = \{D_{ab, a \in \mathrm{AR}^{(k)}(d), b \in \mathrm{AC}^{(k)}(d)} = D_{aj, a \in \mathrm{AR}^{(k)}(d)} = D_{ib, b \in \mathrm{AC}^{(k)}(d)} = d\}$, and $I_{\mathrm{MSNN}}^{(k)} = \{D_{ab, a \in \mathrm{MAR}^{(k)}(d), b \in \mathrm{MAC}^{(k)}(d)} = D_{ib, b \in \mathrm{MAC}^{(k)}(d)}, D_{aj, a \in \mathrm{MAR}^{(k)}(d)} = d\}$, given group $(k)$ and corresponding row/column sets $\mathrm{(M)AR}^{(k)}(d)$ and $\mathrm{(M)AC}^{(k)}(d)$. The number of overlapped rows $\left| \mathrm{(M)AR}^{(1)}(d) \cap \mathrm{(M)AR}^{(2)}(d) \right|$ can vary from $1$ to $r$, while the number of overlapped columns can vary from $0$ to $c$ (despite the case of the two matrices are the same). 


For notation simplicity, we denote $\beta(r^\prime, c^\prime) \coloneqq \binom{m-1-r}{r^\prime} \binom{r}{r^\prime} \binom{n-1-c}{c^\prime} \binom{c}{c^\prime}$. By taking expectation, 

\begin{equation}
  \begin{aligned}
    &\mathbb{E}\left[K_{\rm SNN}^\prime\right] - \mathbb{E}\left[K_{\rm SNN}\right] 
    \le \mathbb{E}\left[K_{\rm SNN}^p\right] \\
    =& \mathbb{E} \left[ \sum_{\substack{\mathrm{AR}^{(1,2)}(d) \subseteq [m]\backslash{i}, \\ \mathrm{AC}^{(1,2)}(d) \subseteq [n]\backslash{j}, \\ 
    \{\mathrm{AR}^{(1)}(d), \mathrm{AC}^{(1)}(d)\} \ne \{\mathrm{AR}^{(2)}(d), \mathrm{AC}^{(2)}(d)\}}} \mathbf{1}\left\{I_{\mathrm{SNN}}^{(1)}, I_{\mathrm{SNN}}^{(2)}\right\} \right] \\ 
    =& \sum_{\substack{\mathrm{AR}^{(1,2)}(d) \subseteq [m]\backslash{i}, \\ \mathrm{AC}^{(1,2)}(d) \subseteq [n]\backslash{j}, \\ 
    \{\mathrm{AR}^{(1)}(d), \mathrm{AC}^{(1)}(d)\} \ne \{\mathrm{AR}^{(2)}(d), \mathrm{AC}^{(2)}(d)\}}} \mathbb{E} \left[ \mathbf{1}\left\{I_{\mathrm{SNN}}^{(1)}, I_{\mathrm{SNN}}^{(2)}\right\} \right] \\
    =& \frac{1}{2} \sum_{\substack{0\le r^\prime \le r-1, \\ 0\le c^\prime \le c, \\ (r^\prime, c^\prime) \ne (0,0)}} \binom{m-1}{r} \binom{n-1}{c} \beta(r^\prime, c^\prime) \cdot p_d^{r+r^\prime} p_d^{c+c^\prime} p_d^{2rc - (r-r^\prime) (c-c^\prime)} , \\
  \end{aligned}
\end{equation}

Similarly, 

\begin{equation}
  \begin{aligned}
    &\mathbb{E}\left[K_{\rm MSNN}^\prime\right] - \mathbb{E}\left[K_{\rm MSNN}\right] 
    \le \mathbb{E}\left[K_{\rm MSNN}^p\right] \\
    =& \mathbb{E} \left[ \sum_{\substack{\mathrm{MAR}^{(1,2)}(d) \subseteq [m]\backslash{i}, \\ \mathrm{MAC}^{(1,2)}(d) \subseteq [n]\backslash{j}, \\ 
    \{\mathrm{MAR}^{(1)}(d), \mathrm{MAC}^{(1)}(d)\} \ne \{\mathrm{MAR}^{(2)}(d), \mathrm{MAC}^{(2)}(d)\}}} \mathbf{1}\left\{I_{\mathrm{MSNN}}^{(1)}, I_{\mathrm{MSNN}}^{(2)}\right\} \right] \\ 
    =& \sum_{\substack{\mathrm{MAR}^{(1,2)}(d) \subseteq [m]\backslash{i}, \\ \mathrm{MAC}^{(1,2)}(d) \subseteq [n]\backslash{j}, \\ 
    \{\mathrm{MAR}^{(1)}(d), \mathrm{MAC}^{(1)}(d)\} \ne \{\mathrm{MAR}^{(2)}(d), \mathrm{MAC}^{(2)}(d)\}}} \mathbb{E} \left[ \mathbf{1}\left\{I_{\mathrm{MSNN}}^{(1)}, I_{\mathrm{MSNN}}^{(2)}\right\} \right] \\
    =& \frac{1}{2} \sum_{\substack{0\le r^\prime \le r-1, \\ 0\le c^\prime \le c, \\ (r^\prime, c^\prime) \ne (0,0)}} \binom{m-1}{r} \binom{n-1}{c} \beta(r^\prime, c^\prime) \cdot p_d^{r+r^\prime} \left( \sum_{d^\prime \in \mathcal{L}} p_{d^\prime}^{r + 1} \right)^{2c^\prime} \left( \sum_{d^\prime \in \mathcal{L}} p_{d^\prime}^{r + r^\prime + 1} \right)^{c-c^\prime} . \\
  \end{aligned}
\end{equation}

From properties of combinatorial number, $\beta(r^\prime, c^\prime) \le \frac{(mr)^{r^\prime} (nc)^{c^\prime}}{(r^\prime)!^2 (c^\prime)!^2}$. 

From the property of the Modified Bessel Function of the First Kind, Order Zero, $I_0(x) \le \frac{1}{1 - x^2/4}$ for $0\le x < 2$. 

By Lemma \ref{lemma: summation of combinatorial numbers, MCAR} and condition (\ref{equation: sparse, relaxed, MCAR}), we have 

\begin{equation}
  \begin{aligned}
    & 2\mathbb{E}\left[K_{\rm SNN}^p\right] / \mathbb{E}\left[K_{\rm SNN}^\prime\right] \\
    =& \sum_{\substack{0\le r^\prime \le r-1, \\ 0\le c^\prime \le c, \\ (r^\prime, c^\prime) \ne (0,0)}} \beta(r^\prime, c^\prime) \cdot p_d^{r^\prime + c^\prime + r^\prime c + r c^\prime - r^\prime c^\prime} \\ 
    \le& \left(I_0 \left(2\sqrt{mr p_d^{c+1}}\right) - 1\right) + \left(I_0\left(2\sqrt{nc p_d^{r+1}}\right) - 1\right) \\+& \left(I_0\left(2\sqrt{mr p_d^{\alpha c + 1}}\right) - 1\right) \left(I_0\left(2\sqrt{nc p_d^{r + 1 + \alpha (1-r)}}\right) - 1\right) \\
    =& O\left( mr p_d^{c+1} + nc p_d^{r+1} + mr nc p_d^{(1 - \alpha)r + \alpha c + 2 + \alpha} \right), 
  \end{aligned}
\end{equation}

and 

\begin{equation}
  \begin{aligned}
    & 2\mathbb{E}\left[K_{\rm MSNN}^p\right] / \mathbb{E}\left[K_{\rm MSNN}^\prime\right] \\
    =& \sum_{\substack{0\le r^\prime \le r-1, \\ 0\le c^\prime \le c, \\ (r^\prime, c^\prime) \ne (0,0)}} \beta(r^\prime, c^\prime) \cdot p_d^{r^\prime} \left( \sum_{d^\prime \in \mathcal{L}} p_{d^\prime}^{r + 1} \right)^{2c^\prime - c} \left( \sum_{d^\prime \in \mathcal{L}} p_{d^\prime}^{r + r^\prime + 1} \right)^{c-c^\prime} \\ 
    \le& \sum_{\substack{0\le r^\prime \le r-1, \\ 0\le c^\prime \le c, \\ (r^\prime, c^\prime) \ne (0,0)}} \beta(r^\prime, c^\prime) \cdot p_d^{r^\prime} \left( \sum_{d^\prime \in \mathcal{L}} p_{d^\prime}^{r + 1} \right)^{c^\prime} p_{\max}^{r^\prime(c-c^\prime)} \\ 
    \le& \left(I_0\left(2\sqrt{mr p_d p_{\max}^{c}}\right) - 1\right) + \left(I_0\left(2\sqrt{nc \left( \sum_{d^\prime \in \mathcal{L}} p_{d^\prime}^{r + 1} \right)}\right) - 1\right) \\+& \left(I_0\left(2\sqrt{mr p_d p_{\max}^{\alpha c}}\right) - 1\right) \left(I_0\left(2\sqrt{nc \left( \sum_{d^\prime \in \mathcal{L}} p_{d^\prime}^{r + 1} \right) p_{\max}^{\alpha (1-r)}} \right) - 1\right) \\
    =& O\left( mr p_d p_{\max}^{c} + nc\gamma p_{\max}^{r + 1} + mr nc\gamma p_d p_{\max}^{\alpha c + (1 - \alpha)r + 1 + \alpha} \right). 
  \end{aligned}
\end{equation}

By combining the results above, under (\ref{equation: sparse, relaxed, MCAR}) we have 

\begin{equation}
  \begin{aligned}
    \frac{\mathbb{E}\left[K_{\rm SNN}(d)\right]}{\binom{m-1}{r}\binom{n-1}{c} p_d^{rc + r + c}} &\in \left[ \frac{1}{1 + O\left( mr p_d^{c+1} + nc p_d^{r+1} + mr nc p_d^{(1 - \alpha)r + \alpha c + 2 + \alpha} \right)} , 1 \right] , \\
    \frac{\mathbb{E}\left[K_{\rm MSNN}(d)\right]}{\binom{m-1}{r}\binom{n-1}{c} \gamma^c p_d^{r} p_{\max}^{(r+1)c}} &\in \left[ \frac{1}{1 + O\left( mr p_d p_{\max}^{c} + nc\gamma p_{\max}^{r + 1} + mr nc\gamma p_d p_{\max}^{\alpha c + (1 - \alpha)r + 1 + \alpha} \right)} , 1 \right] . 
  \end{aligned}
\end{equation}

Further by the sparsity condition in Theorem \ref{theorem: expectation of K, MCAR}, 

\begin{equation}
  \begin{aligned}
    O\left( mr p_d^{c+1} + nc p_d^{r+1} + mr nc p_d^{(1 - \alpha)r + \alpha c + 2 + \alpha} \right) &= o(1) , \\ 
    O\left( mr p_d p_{\max}^{c} + nc\gamma p_{\max}^{r + 1} + mr nc\gamma p_d p_{\max}^{\alpha c + (1 - \alpha)r + 1 + \alpha} \right) &= o(1) . 
  \end{aligned}
\end{equation}

This completes the proof. 

\end{proof}

Proof of Corollary \ref{corollary: data efficiency, comparison to SNN, MCAR}: 

\begin{proof}\label{proof: corollary: data efficiency, comparison to SNN, MCAR}

By Theorem \ref{theorem: expectation of K, MCAR} we have 

\begin{equation}
  \begin{aligned}
    \mathbb{E}\left[K_{\rm SNN}(d)\right] &= (1 - o(1)) \binom{m-1}{r}\binom{n-1}{c} p_d^{rc + r + c} , \\ 
    \mathbb{E}\left[K_{\rm MSNN}(d)\right] &= (1 - o(1)) \binom{m-1}{r}\binom{n-1}{c} \gamma^c p_d^{r} p_{\max}^{(r+1)c} . 
  \end{aligned}
\end{equation}

By taking division we completes the proof. 

\end{proof}

Proof of Corollary \ref{corollary: efficiency compared to the treatment level with the most data, MCAR}: 

\begin{proof}\label{proof: corollary: efficiency compared to the treatment level with the most data, MCAR}


From Theorem \ref{theorem: expectation of K, MCAR}, 

\begin{equation}
  \begin{aligned}
    \mathbb{E}\left[K_{\rm SNN}(d)\right] &= (1 - o(1)) \binom{m-1}{r}\binom{n-1}{c} p_d^{rc + r + c} , \\
    \mathbb{E}\left[K_{\rm MSNN}(d)\right] &= (1 - o(1)) \binom{m-1}{r}\binom{n-1}{c} \gamma^c p_d^{r} p_{\max}^{(r+1)c} , \\
    \mathbb{E}\left[K_{\rm MSNN}(d_{\max})\right] &\le \binom{m-1}{r}\binom{n-1}{c} \gamma^c p_{\max}^{rc + r + c} . 
  \end{aligned}
\end{equation}



By taking division, we completes the proof. 

\end{proof}


\end{document}